\documentclass{article}
\usepackage[accepted]{icml2024}
\usepackage[toc,page,header]{appendix}
\usepackage{minitoc}
\usepackage{url} 
\usepackage{booktabs} 
\usepackage{amsfonts} 
\usepackage{nicefrac} 
\usepackage{microtype}
\usepackage{graphicx}
\usepackage{wrapfig}
\usepackage{subcaption}
\usepackage{booktabs} 
\usepackage{def}
\usepackage{xcolor}
\usepackage{xspace}
\usepackage{algorithm}
\usepackage{algorithmic}
\usepackage{amssymb}
\usepackage{enumitem}
\usepackage{comment}
\usepackage{bm}
\usepackage{bbm}
\usepackage{breqn}
\usepackage{authblk}
\usepackage{mysymbol}
\usepackage{mdframed} 
\usepackage{amsmath,amsthm}
\definecolor{linen}{RGB}{250,240,230} 
\usepackage[textsize=tiny]{todonotes}
\mdfdefinestyle{theoremstyle}{
  linecolor=blue,
  linewidth=0pt,
  backgroundcolor=linen,
}
\mdfdefinestyle{remarkstyle}{
  linecolor=blue,
  linewidth=0pt,
  backgroundcolor=lightblue,
}
\mdfdefinestyle{defstyle}{
  linecolor=blue,
  linewidth=0pt,
  backgroundcolor=lightgray,
}
\usepackage[colorlinks=true, linkcolor=blue, citecolor=blue]{hyperref}

\usepackage{xcolor}
\hypersetup{
    colorlinks,
    linkcolor={blue!50!black},
    citecolor={blue!50!black},
    urlcolor={blue!80!black}
}
\newmdtheoremenv[style=theoremstyle]{thm}{Theorem}
\newmdtheoremenv[style=theoremstyle]{lem}{Lemma}
\newmdtheoremenv[style=defstyle]{ass}{Assumption}
\newmdtheoremenv[style=remarkstyle]{rem}{Remark}
\newmdtheoremenv[style=remarkstyle]{cor}{Corollary}
\newmdtheoremenv[style=theoremstyle]{defi}{Definition}

\allowdisplaybreaks
\usepackage[scaled=.9]{helvet}
\definecolor{darkgreen}{rgb}{0,0.5,0}
\definecolor{darkred}{rgb}{0.7,0,0}
\definecolor{teal}{rgb}{0.3,0.8,0.8}
\definecolor{orange}{rgb}{1.0,0.5,0.0}
\definecolor{purple}{rgb}{0.8,0.0,0.8}

\newcommand{\piref}{\pi_\text{ref}}
\newcommand{\pisft}{\pi_\text{ref}} %
\icmltitlerunning{MaxMin Approach to Align with Diverse Human Preferences}

\begin{document}

\twocolumn[
\icmltitle{MaxMin-RLHF: Alignment with Diverse Human Preferences}

\icmlsetsymbol{equal}{*}

\begin{icmlauthorlist}
\icmlauthor{Souradip Chakraborty}{equal,umd}
\icmlauthor{Jiahao Qiu}{equal,princ}
\icmlauthor{Hui Yuan}{princ}
\icmlauthor{Alec Koppel}{jpmc}
\icmlauthor{Dinesh Manocha}{umd}
\icmlauthor{Furong Huang}{umd}
\icmlauthor{Amrit Singh Bedi}{ucf}
\icmlauthor{Mengdi Wang}{princ}
\end{icmlauthorlist}

\icmlaffiliation{umd}{Department of Computer Science, University of Maryland, College Park, MD,  USA.}
\icmlaffiliation{princ}{Department of Electrical and Computer Engineering, Princeton University, NJ, USA.}
\icmlaffiliation{jpmc}{JP Morgan Chase AI Research,  New York, USA.}
\icmlaffiliation{ucf}{Department of Computer Science,, University of Central Florida, FL, USA.}

\icmlcorrespondingauthor{Souradip Chakraborty}{schakra3@umd.edu}

\icmlkeywords{Machine Learning, ICML}

\vskip 0.3in
]

\printAffiliationsAndNotice{\icmlEqualContribution} 
\begin{abstract}
Reinforcement Learning from Human Feedback (RLHF) aligns language models to human preferences by employing a singular reward model derived from preference data.  However, the single reward model overlooks the rich diversity of human preferences inherent in data collected from multiple users. In this work, we first derive an impossibility result of alignment with single reward RLHF, thereby highlighting its insufficiency in representing diverse human preferences. Next, we propose to learn a mixture of reward models via an expectation-maximization algorithm and solve a MaxMin alignment objective inspired by the Egalitarian principle in social choice theory to better honor diverse human preferences. We present comprehensive experimental results on small-scale (GPT-2) and large-scale language (with Tulu2-7B)) and show the efficacy of the proposed approach in the presence of diversity among human preferences.  We remark that our findings in this work are not only limited to language models but also extend to reinforcement learning in general. 

\end{abstract}
\section{Introduction}
 The alignment problem, central to developing and fine-tuning current large language models (LLMs), represents a crucial challenge in artificial intelligence,  especially in ensuring these models operate in harmony with human values and preferences \citep{wang2023aligning, christian2020alignment}. Reinforcement learning from human feedback (RLHF) has emerged as a pivotal approach to alignment problems, specifically aligning LLM \citep{wang2023aligning, ouyang2022training, openai1, openai2}.  RLHF starts with pre-training a generative (LLM) model and subsequently fine-tuning it through supervised learning on a high-quality dataset on various downstream tasks. RLHF operates in three steps (a) supervised fine-tuning, (2) reward learning, and (3) RL fine-tuning. Step 2 learns a reward function that is expected to represent the preference feedback of the human population.  

However, there has been minimal emphasis on accurately representing the diversity of human preferences and the broad spectrum of user populations. As highlighted by \citet{diversity1,aroyo2023reasonable, aroyo2023dices}, \textit{``the notion of `one truth' in crowdsourcing responses is a myth"} and we need to account for the diversity in opinions and preferences. 

Despite the criticality, most of the latest RLHF approaches ignore the consideration of the diversity in human preference feedback by aligning the language model with a single reward \citep{wang2023aligning, christian2020alignment,openai1, openai2}. The assumption of a single ground truth reward is restrictive and can potentially subdue the preferences or opinions of minority groups, leading to societal biases (Figure \ref{fig:bias_RLHF}). To mitigate this issue, some of the recent research proposes to learn multiple reward functions, which can then be aggregated in arbitrary manners \citep{agg1}.  On the other hand, \citep{agg2} adopts a consensus-based method for aggregating human representations by emphasizing specific principles \citep{bias1, bias2}, which might result in the under-representation of marginalized groups \citep{ramé2023rewarded}. Another line of research focuses on the aspect of designing multi-policy strategies by fine-tuning personalized language models towards individual rewards \citep{jang2023personalized, ramé2023rewarded, ji2023beavertails}. 

As mentioned above, the recent literature has brought attention to the challenge of aligning single utility RLHF with diverse preferences. However, a thorough understanding of how the diversity within human sub-populations influences the overall alignment objective remains elusive. Consequently, this prompts us to pose the following question:
\emph{Is a single reward RLHF pipeline sufficient to align with diverse human preferences?}

In this work, we present negative results for the above question in this work by demonstrating the impossibility of alignment using single reward RLHF (Theorem \ref{theorem_2_impossibility}).  We introduce a notion of diversity between human subpopulations due to the differences in preference distributions and establish lower bounds on the alignment performance of single reward RLHF. However, this impossibility result naturally {raises} another important question:

{\emph{What strategies can we design (or what methods can we adopt) to align with diverse human preferences?}}

In response to this question, we draw inspiration from the Egalitarian rule \cite{Sen2017} and aim to maximize the social utility objective for alignment.  We summarize our contributions as follows.
\begin{figure*}[t]
    \centering
\includegraphics[scale=0.5]{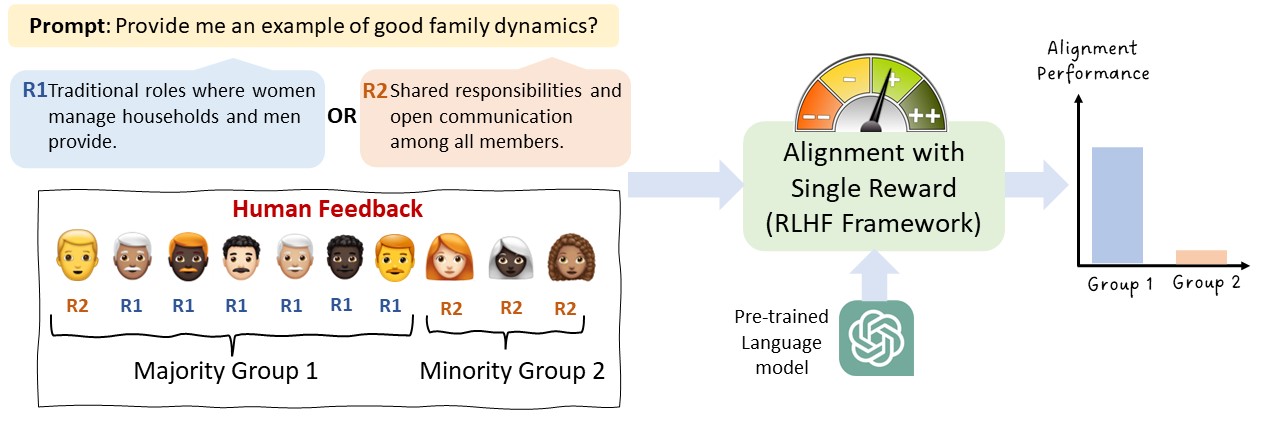}
\vspace{-0mm}
    \caption{This figure highlights the drawbacks of a single reward-based current state-of-the-art alignment framework called Reinforcement Learning from Human Feedback (RLHF) \citep{christian2020alignment}. In this figure, we demonstrate a setting where, due to the inherent presence of majority and minority user groups who provide human feedback, single reward-based RLHF alignment would align the language model towards the majority group while completely ignoring the minority use group preferences. We provide a theoretical justification in Section \ref{impossibility} and empirical evidence in Section \ref{experiments}.}
    \label{fig:bias_RLHF}
\end{figure*}

 \textbf{(1) An impossibility result of alignment with single reward-based RLHF.} We first introduce the notation of diversity  (Definition \ref{diver}) and then derive lower bounds on the reward model suboptimality (Lemma \ref{theorem_1_impossibility0}) in terms of diversity in human sub-population preference distributions. Finally, we establish a lower bound (Theorem \ref{theorem_2_impossibility}) on the alignment gap due to the diversity in the human preference feedback.  True to our knowledge, our work is the first to report such a result in the RLHF literature. 
    
  \textbf{(2) Max-Min RLHF alignment with diverse user preferences.} We propose to learn a mixture of preference distributions through the application of multiple reward functions using the Expectation-Maximization (EM) algorithm (Algorithm \ref{alg:cdm}). Upon obtaining multiple reward functions specific to different human sub-populations, we introduce the MaxMin-RLHF algorithm as a strategy to align language models with social utility objectives (Algorithm \ref{alg:cdm2}).
  
 \textbf{(3) A comprehensive empirical study. } 
We present a detailed empirical analysis of our proposed concepts on two language models: GPT-2 and Tulu-7B. Initially, we provide empirical evidence highlighting the impossibilities of alignment with single reward RLHF, followed by demonstrating the feasibility and effectiveness of MaxMin-RLHF in achieving social utility objectives. Our approach outperforms existing methodologies, showcasing significant performance improvements.

\section{Preliminaries}\label{section:prelims}

Let us start by defining a language model mathematically. We denote a vocabulary set as $\mathcal{V}$ and a language model by a mapping $\pi_{\theta}$ (parameterized by $\theta$).  A language model $\pi_{\theta}$ takes a sequence of tokens (called prompt) as input denoted by $\mathbf{x} := \{x_{1}, x_{2}, \cdots, x_{N}\}$, where each token $x_i \in \mathcal{V}$. The prompt $\mathbf{x}\in\mathcal{X}$, where $\mathcal{X}$ is the set of prompts,  is fed as input to the language model, and it generates output response  $\mathbf{y}\sim \pi_{\theta}(\cdot~|~\mathbf{x})$.

\noindent \textbf{RLHF pipeline.} We start by considering the RLHF pipeline in \citet{ziegler2020finetuning}, which has also been adopted in subsequent works \citep{stiennon2022learning, bai2022training, ouyang2022training}. It consists of three steps detailed as follows: 

\noindent\textbf{Step 1: Supervised Fine-tuning (SFT)}: In this phase, a generic pre-trained LM is fine-tuned with supervised learning on a high-quality dataset for the downstream task(s) of interest, such as dialogue, instruction following, summarization, etc., to obtain a model $\pisft$. 

\noindent \textbf{Step 2: Reward Modelling}: In the second phase, the SFT model is queried with prompts $\mathbf{x}\in\mathcal{X}$ to produce pairs of responses $(\mathbf{y}_1, \mathbf{y}_2)\sim \pi_{\theta}(\cdot~|~\mathbf{x}) $ which are then presented to human labelers for preference evaluation, and $\mathbf y_1$, $\mathbf y_2$ denotes the preferred and dispreferred response, respectively. The preference distribution under the Bradley-Terry (BT) preference model \citep{bradley1952rankanalysis} is written as
\begin{equation}\label{eq:bradley-terry}
    p^*(\mathbf y_1 \!\succ\! \mathbf y_2 \mid \mathbf x)\!=\!\frac{\exp\left(r^*(\mathbf y_1, \mathbf x)\right)}{\exp\left(r^*(\mathbf y_1, \mathbf x)\right) + \exp\left(r^*(\mathbf y_2, \mathbf x)\right)},
\end{equation}
where $r^*(\mathbf y, \mathbf x)$ is the latent reward model. With a static dataset $\mathcal{D}=\bigl\{\mathbf x^{(i)}, \mathbf y_1^{(i)}, \mathbf y_2^{(i)}\bigr\}_{i=1}^N$ sampled from $p^*$, we can learn a parameterized reward model $r_{\phi}(\mathbf y, \mathbf x)$ via maximum likelihood estimation. Framing the problem as a binary classification, we have the negative log-likelihood loss:
\begin{equation}\label{eq:reward_model}
    \mathcal{L}_R(r_{\phi}, \mathcal{D}) \!=\! -\mathbb{E}_{(\mathbf x, \mathbf y_1, \mathbf y_2)\sim \mathcal{D}}\bigl[\log \sigma(r_{\phi}(\mathbf y_1, \mathbf x)- r_{\phi}(\mathbf y_2, \mathbf x))\bigr]
\end{equation}
where $\sigma$ is the logistic function. 

\noindent \textbf{Step 3: RL Fine-Tuning}: In the final step, the optimal policy $\pi_{r_{\phi}}^*$ under the reward $r_{\phi}$  is obtained by solving the KL-regularized reward maximization problem given by
\begin{align}\label{eq:RL}
\max_{{\pi}}  \mathbb{E}_{\mathbf x\sim \mathcal{P},\mathbf{y}\sim \pi(\cdot~|~\mathbf{x}) }\bigl[r_{\phi}(\mathbf y, \mathbf x) - \beta\mathbb{D}_{\textrm{KL}}\bigl[\pi(\cdot|\mathbf{x})|| \pi_{\text{ref}}(\cdot|\mathbf{x})\bigr]\bigr], 
\end{align}
where, $\beta>0$ controls the deviation from the base reference policy $\piref$.

\section{An Impossibility Result for  Single Reward RLHF with Diverse Preferences}\label{impossibility}
In this section, we mathematically prove the impossibility of aligning language models with diverse human preferences with the single reward RLHF framework. We start by discussing the motivation and mathematical definition of diversity in human preferences in Section \ref{sub1}, then connect the reward learning step of the RLHF pipeline with diversity in Section \ref{reward_2earning}, and then finally prove the impossibility of language model alignment in Section \ref{sub3} by connecting Step 3 of RLHF pipeline with human preference diversity. 
\subsection{Diversity in Human Preferences}\label{sub1}
\begin{figure}[t]
    \centering
    \includegraphics[scale=0.9]{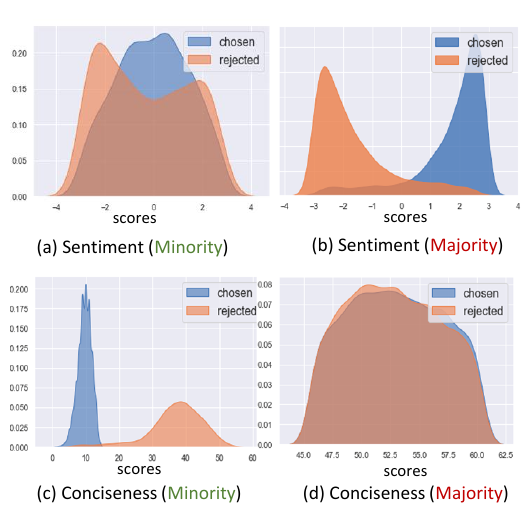}
    \caption{\textbf{(Diversity in Preferences.)} This figure illustrates the diversity in preferences among two distinct human groups using the IMDB movie review dataset \citep{maas-EtAl:2011:ACL-HLT2011}. We categorize these groups as `majority' and `minority.' 
(a) and (c) display minority sentiment and conciseness preferences. We note that the minority group strongly favors concise responses (as seen in the blue curve in (c)), while showing indifference towards sentiment (as indicated by overlapping curves in (a)).
In contrast, (b) and (d) depict that the majority clearly prioritizes positive sentiment (as evidenced by a significant gap between chosen and rejected trajectories in (b)), while displaying little concern for conciseness (as indicated by overlapping curves in (d)).}
    \label{fig:diversity}
\end{figure}
The main shortcoming of state-of-the-art alignment approaches arises from the underlying assumption that human preferences are derived from a single latent reward model $r^*(\textbf{y},\textbf{x})$ (cf. \eqref{eq:reward_model}), which fails to account for the inherent diversity among the human sub-populations (see Figure \ref{fig:diversity}).  One of the key reasons (discussed in Appendix \ref{related_works_appendix} in ) for the diverse human preferences is the varied socio-demographic and socio-cultural backgrounds of human sub-populations \citep{aroyo2023dices,aroyo2023reasonable}. For example, population groups with diverse demographic markers such as race, ethnicity, age groups, genders, etc., have highly varied preferences as highlighted in \citep{aroyo2023dices, aroyo2023reasonable, denton2021ground}. Such diversity inevitably leads to natural sub-groups of populations among humans. Modeling this diversity in preferences for the fine-tuning of language models in RLHF is crucial, which, to the best of our knowledge, is currently missing from the literature. 

\noindent \textbf{Sub-population Preference Distributions:} 
Let us consider the human population providing the preference feedback represented by $\mathcal{H}$. We can write the preference  distribution \citep{openai1, openai2} as %
\begin{align}\label{h1}
   p^*(&\mathbf y_1 \succ \mathbf y_2 \mid \mathbf x)
   \\
   &=\mathbb{E}_{h\in\mathcal{H}} [\mathbb{I} (\texttt{h prefers y_1 over y_2|x)}], \nonumber
\end{align}
 where $p^*(\mathbf y_1 \succ \mathbf y_2 \mid \mathbf x)$ is the probability of preferring $\bby_1$ over $\bby_2$ for any given pair $(\mathbf y_1 ,\mathbf y_2)$ corresponding to prompt $\mathbf x$. In \eqref{h1}, the expectation is over a finite set of humans $h\in{\mathcal{H}}$. We next introduce the concept of human subpopulations as a hidden random variable, denoted as $u$ with distribution $\eta$, to account for the inherent diversity within the population. Specifically, $u$ represents the human subpopulation defined over a finite discrete set $\mathcal{U} := \{\mathcal{H}_1, \mathcal{H}_2, \cdots, \mathcal{H}_{|\mathcal{U}|}\}$, such that $\mathcal{H} = \bigcup_{u =1}^{|\mathcal{U}|} \mathcal{H}_u$. The cardinality of the set $\mathcal{U}$ represents the number of sub-populations/groups present in the total human population $\mathcal{H}$.

 Therefore, similar to \eqref{h1}, we can define a human-subpopulation or group-specific preference distribution for a given pair of responses $(\mathbf y_1,\mathbf y_2)$ and prompt $\mathbf x$ as 
 \begin{align}\label{group_specific}
       p^*_u(&\mathbf y_1 \succ \mathbf y_2 \mid \mathbf x)
   \\
   &=\mathbb{E}_{h\in\mathcal{H}_u} [\mathbb{I} (\texttt{h prefers y_1 over y_2|x)}], \nonumber
 \end{align}
 for all groups in $\mathcal{U}$. Next, we define the preference diversity among the human population in Definition \ref{diver} as follows. 

 \begin{defi}[Diversity in Human Preferences]
 \label{diver}
 \vspace{2mm}
     Consider a human population $\mathcal{H}$, composed of $|\mathcal{U}|$ sub-population groups where $\mathcal{H} = \bigcup_{u =1}^{|\mathcal{U}|} \mathcal{H}_u$, and a sub-population-specific preference $p^*_u$ as defined in \eqref{group_specific}, we define the diversity of sub-population group $\mathcal{H}_i$ with respect to other group $\mathcal{H}_j$ as%
     \begin{align}
         \texttt{Diversity}\ (i,j):= \text{TV}(p_i^*, p^*_j),
     \end{align}   %
     where $\text{TV}$ denotes the total variation distance between two preference distributions.
 \end{defi}
  
{By utilizing the definition of sub-population groups in $\mathcal{U}$, we can express the preference in \eqref{h1} as}
\begin{align}\label{new_definition}
       p^*(\mathbf y_1 \succ \mathbf y_2 \mid \mathbf x)  = &  \sum_{u=1}^{ |\mathcal{U}|}  \bigg[\sum_{h \in \mathcal{H}_{u}}\mathbb{I}_h (\mathbf z) \cdot q(h|u)\bigg]\cdot  \eta(u) 
      \nonumber \\
      =&\sum_{u=1}^{ |\mathcal{U}|} p_u^*(\bbz) \cdot \eta(u),
\end{align}
 where $\mathbf z:=(\mathbf y_1 \succ \mathbf y_2 \mid \mathbf x)$ is a shorthand notation and $q(\cdot)$ denotes the distribution over the humans $\mathcal{H}$. Here,  $p_u^*(\bbz_h) = \sum_{h \in \mathcal{H}_{u}}\mathbb{I}_h (\mathbf z) \cdot q(h|u)$ is the sub-population specific preference distribution (cf. \eqref{group_specific}) and  $\eta(\cdot)$ represents the marginal probability distribution of sub-population $\mathcal{H}_u$ and quantifies the probability of occurrence of sub-population $\mathcal{H}_u$ to provide feedback for pair $\textbf{z}$. We can think of $\eta(\cdot)$ as a weighting function that quantifies the relative importance of each sub-population (say $\mathcal{H}_u$) within the full population $\mathcal{H}$ reflecting their contributions to the aggregate preference distribution $p^*$. 
{Thus, from the expansion in \eqref{new_definition}, it is evident that the preference distribution under consideration is a weighted sum of sub-population specific preference distribution, weighted by $\eta(u)$}. 
We remark that distributions $q$ and $\eta$ are crucial to rigorously characterize the alignment performance of different approaches, which is not considered in the existing literature \citep{christian2020alignment,bai2022training}.

\subsection{Reward Mismatch Due to Diversity} \label{reward_2earning}
From equations \eqref{eq:bradley-terry} and \eqref{eq:reward_model}, we note that the existing RLHF approach focuses on learning the ground-truth single reward parameter $\phi^*$ to represent the preference distribution $p^*$ by minimizing the cross-entropy loss (cf. \eqref{eq:reward_model})  given by
\begin{align}\label{loss}
    \mathcal{L}_R(r_{\phi}, \mathcal{D})= &- \mathbb{E}_{(\mathbf x, \mathbf y_1, \mathbf y_2)\sim \mathcal{D}}\Big[p^*(\mathbf y_1\succ\mathbf y_2 \mid \mathbf x) \log p_{\phi}(\succ) \nonumber \\ 
    & + p^*(\mathbf y_1 \prec \mathbf y_2 \mid \mathbf x) \log p_{\phi}(\prec) \Big],
\end{align}
{The assumption of single ground-truth reward (corresponding to $p^*$) which is violated due to the existence of diverse sub-populations with separate preference distributions, as discussed in Section \ref{sub1}. This would lead to an implicit aggregation as shown in \eqref{new_definition} and the equivalent MLE objective in \eqref{loss} can be re-written as :}
\begin{align}\label{mle_upd}
    \mathcal{L}_R(r_{\phi}, \mathcal{D})&\!=\! - \mathbb{E}_{(\mathbf x, \mathbf y_1, \mathbf y_2)\sim \mathcal{D}}\Big[ \mathbb{E}_{u} [p^*_u(\mathbf y_1\! \succ\! \mathbf y_2\mid\! \mathbf x)] \log p_{\phi}(\succ) \nonumber\\ 
    & + \mathbb{E}_{u} [p^*_u(\mathbf y_1 \prec \mathbf y_2 \mid \mathbf x)] \log p_{\phi}(\prec) \Big].
\end{align}
{Now, expanding upon the cross-entropy objective, we note (see Lemma \ref{lemma_1} for details) that the objective in \eqref{mle_upd} essentially reduces to minimizing the Kullback-Leibler (KL) divergence $\mathsf{KL} (\sum_{u=1}^{ |\mathcal{U}|} \eta(u)  p_u^* (\bbz) || p_{\phi})$
and the objective is minimized at $p_{\phi^*} = \sum_{u=1}^{ |\mathcal{U}|} \eta(u)  p_u^*$. This implies that by minimizing the loss function in \eqref{mle_upd}, when we try to learn a single $\phi^*$ to recover $p^*$, an implicit averaging happens over the preferences of human subpopulation groups they belong to, which plays a critical role in the sub-optimality in reward learning summarized in Lemma \ref{theorem_1_impossibility0}.}
\begin{lem}[]\label{theorem_1_impossibility0}
\vspace{2mm}
    Let $\phi^*$ denotes the reward parameter, which models $p^*$ (cf. \ref{eq:bradley-terry}) and $\phi^*_u$ models the human sub-population group $\mathcal{H}_u\in\mathcal{U}$ specific $p_u^*$, it holds that 
    \begin{align}
      \underbrace{\|\phi^* - \phi_u^*\|}_{\textbf{Reward mismatch}}\geq  \frac{\epsilon (1-\eta(u))}{4D}, 
    \end{align}
    where $\epsilon:=\texttt{Diverity} (u, j) - \max_{k\neq u} \texttt{Diverity} (k, j) >0$, $D$ denotes the upper bound on the feature representation $\|\psi(\bby,\bbx)\|\leq D$ for all $(\bbx,\bby)$, and diversity as defined in Definition \ref{diver}.
\end{lem} 
\textbf{Proof Sketch.} Here we describe the proof sketch of Lemma \ref{theorem_1_impossibility0} with a detailed proof provided in Appendix \ref{theorem_1_impossibility_proof}. We begin with the definition of sub-optimality in the learned reward for a subpopulation group $u$ as $\Delta_u^r := \hat{\phi}_{\texttt{MLE}} - \phi_u^*$ where $\hat{\phi}_{\texttt{MLE}}$ which is the approximation to the true parameter $\phi^*$. However, we know in the limit of infinite data under appropriate regulatory conditions, $\hat{\phi}_{\texttt{MLE}}$ converges to $\phi^*$, and hence we focus on the sub-optimality gap due to diversity as $\|\phi_u^* - \phi^*\|$. 
 Using the Lipschitzness of the preference probability distribution under the Bradley-Tery preference model (derived in Lemma \ref{lemma_2} in Appendix) we lower-bound the sub-optimality gap and finally expanding upon the definition of $p^*$ as shown in \eqref{new_definition}, we get the final result. 
 
\textbf{Remark.} Lemma \ref{theorem_1_impossibility0} indicates that the current RLHF-based reward learning paradigm \citep{christian2020alignment,bai2022training,rafailov2023direct} will suffer sub-optimality due to diversity amongst the humans, which is highly likely in practice \citep{aroyo2023dices}. Lemma \ref{theorem_1_impossibility0} implies that the degree to which the learned reward parameter diverges from optimality for a given subgroup is influenced by two key factors: the distinctiveness of that subgroup’s preferences compared to all the other subgroups, and the relative weight assigned to the subgroup in the overall preference model. 

\subsection{An Impossibility Results of Alignment}\label{sub3}
To mathematically characterize the impossibility of aligning the language model with diverse sub-population groups,  let us reconsider the RL fine-tuning optimization problem, which is given by (step 3 in RLHF) 
\begin{align}\label{eq:RL2}
\max_{{\pi}}  F_{r_{\phi}}(\pi),
\end{align}
\begin{figure}[t]
    \centering
    \includegraphics[width=\columnwidth]{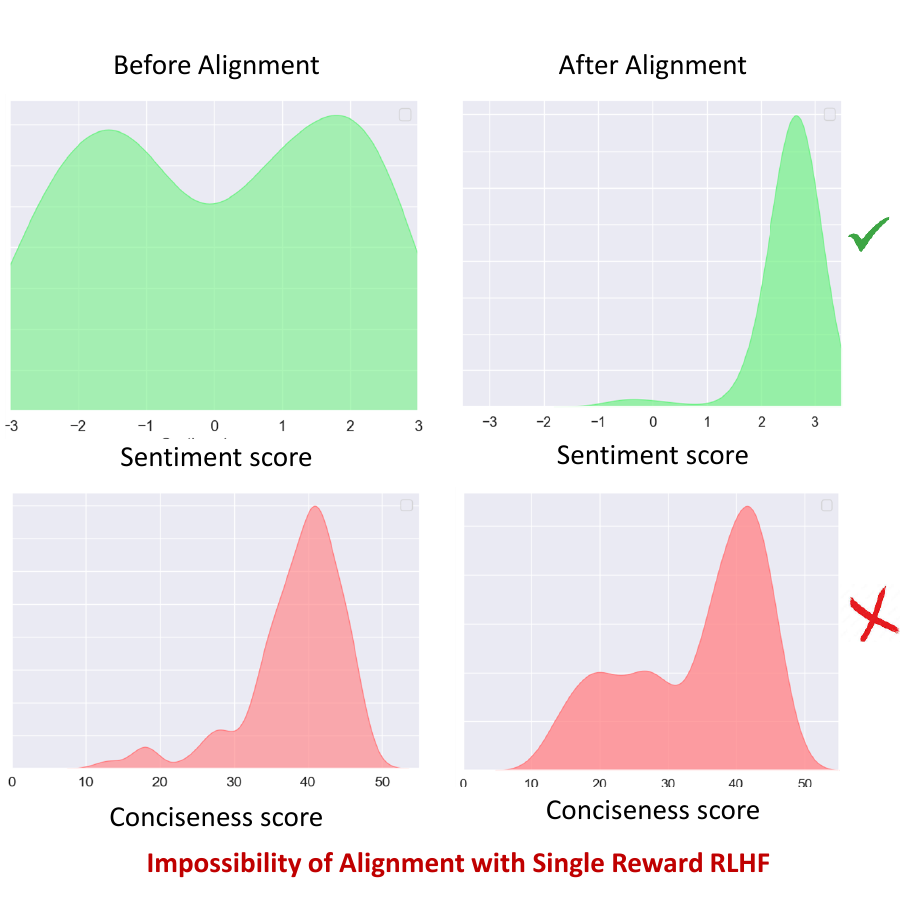}\vspace{-3mm}
    \caption{\textbf{(Empirical Evidence of Impossibility).} This figure validates our theoretical results in Theorem \ref{theorem_2_impossibility} and provides empirical evidence of the impossibility of alignment in single reward RLHF on preference dataset presented in Figure \ref{fig:diversity}. Here, the task is to align the LLM to generate positive sentiment responses which are concise. We note that the aligned language model can generate highly positive sentiment sentences but completely ignores the requirement of conciseness. This is happening because the humans who prefer conciseness are in the minority as compared to humans who prefer a positive sentiment score as described in Figure \ref{fig:diversity}. }
    \label{fig:impossibility}
\end{figure}
where we define $F_{r_{\phi}}(\pi):=\mathbb{E}_{\mathbf x\sim \mathcal{P}} \Big[\mathbb{E}_{ \mathbf{y}\sim \pi(\cdot~|~\mathbf{x})}\bigl[r_{\phi}(\mathbf y, \mathbf x)\bigr] - \beta\mathbb{D}_{\textrm{KL}}\bigl[\pi(\cdot~|~\mathbf{x})|| \pi_{\text{ref}}(\cdot~|~\mathbf{x})\bigr]\Big]$. Let us define  $\pi^*_{\text{RLHF}}:=\argmax_{{\pi}}  F_{r_{\phi^*}}(\pi)$
    where $\pi^*_{\text{RLHF}}$ is the optimal aligned policy with single reward RLHF. On the other hand, we define a human sub-population specific optimal policy as $ \pi^*_{\text{u}}:=\argmax_{{\pi}}  F_{r_{\phi^*_u}}(\pi)\label{pi_u}$, 
 where $\pi^*_{\text{u}}$ is the optimal aligned policy with individual subpopulation group $\mathcal{H}_u$. We define the alignment gap of RLHF model $  \pi^*_{\text{RLHF}}$ to a specific  user group $\mathcal{H}_u$  by
\begin{align}\label{align_gap}
    \text{Align-Gap}(\pi_{\text{RLHF}}): =F_{r_{\phi_u^*}}(\pi_u^*)-F_{r_{\phi_u^*}}(\pi_{\text{RLHF}}).
\end{align}
We note that the alignment gap defined in \eqref{align_gap} measures the discrepancy between the reward returns by the single reward RLHF model $\pi_{\text{RLHF}}$ and the optimal model $\pi_u^*$ tailored for $\mathcal{H}_u$ subpopulation evaluated under true reward function $r_u^*$. 
Next,  we present our impossibility result in Theorem \ref{theorem_2_impossibility}
\begin{thm}[An Impossibility Result]\label{theorem_2_impossibility}
\vspace{2mm}
    Let $\phi^*$ denotes the reward parameter, which models $p^*$ (cf. \ref{eq:bradley-terry}),  $\phi^*_u$ denotes the human sub-population-group $\mathcal{H}_u\in\mathcal{U}$ specific reward function to model $p_u^*$, and alignment gap is as defined in \eqref{align_gap}.
    Then, it holds that 
    \begin{align}
    \text{Align-Gap}  \geq & \frac{ \lambda_{\psi} }{64 \beta^2L_{\pi}}\cdot \frac{\epsilon (1 -\eta(u))}{D^2},
    \end{align}
    where $\epsilon:=\texttt{Diversity} (u, j) - \max_{k\neq u} \texttt{Diversity} (k, j) >0$,  $\eta$ denotes the representation for the human sub-population group $u$,   $D$ denotes the upper bound on the feature representation $\|\psi(\bby,\bbx)\|\leq D$ for all $(\bbx,\bby)$, $\lambda_{\psi}$ denotes the minimum eigenvalue of the feature matrix, $\beta$ is the regularization parameter of RLHF framework and diversity as defined in Definition \ref{diver}.
\end{thm} 
A detailed proof of Theorem \ref{theorem_2_impossibility} is provided in Appendix \ref{proof_theoem_1}.  We briefly describe the proof sketch of Theorem \ref{theorem_2_impossibility} as follows.

\textbf{Proof Sketch.} 
We begin by considering the KL-regularized alignment objective (cf. \eqref{eq:RL}). Utilizing the strong concavity of the objective under the KL regularization and the analytical mapping from reward functions to optimal policies (as used in DPO \citep{rafailov2023direct}), we first derive a lower bound on the alignment gap as $\text{Align-Gap}(\pi_{\text{RLHF}})  \geq  \frac{ 1}{2 L_{\pi} \beta^2} \|r_{\phi^*}-r_{\phi_u^*}\|^2$. Under the linear parametrization in reward and utilizing the boundedness on the representation space, we can lower-bound the alignment gap with the reward sub-optimality and eventually the diversity coefficient. 

\textbf{Remark.} Theorem \ref{theorem_2_impossibility} shows that high subpopulation diversity inevitably leads to a greater alignment gap. {Here, $\epsilon$ depends on the diversity among user groups, highlighting that when the diversity between a specific user group $u$ and others is significantly higher compared to inter-group diversity, it suggests that $u$ is a minority. Consequently, aligning to this particular user-group with single-reward RLHF becomes particularly challenging. Moreover, as the representation of the user-group reduces i.e $\eta(u) \to 0$, the alignment gap further increases making it harder to align to this user-group.} In summary, if a subgroup exhibits distinctive preferences or constitutes a minority with a smaller representation, the resulting model from single reward RLHF setting cannot accurately reflect the sub-population's specific preferences. We provide empirical evidence of the impossibility of alignment in Figure \ref{fig:enter-label}.

\section{MaxMin-RLHF: One Possibility}

From the statement of Theorem \ref{theorem_2_impossibility}, it is clear that it is not possible to align diverse human preferences with a single reward RLHF. We start by noting that even if we can bypass the sub-optimality in reward learning (cf. Lemma \ref{theorem_1_impossibility0}) by learning multiple reward functions $\hat{\phi}_u$ for all $\mathcal{H}_u$, it doesn't resolve the eventually aim of language model alignment. This is because our goal is to develop a single model $\pi^*$ that honors diverse user preferences without demonstrating bias towards specific groups such as minorities. 
To achieve that, we take motivation from the Egalitarian rule in social choice theory \cite{Sen2017}, which states that society should focus on maximizing the minimum utility of all individuals. 
Hence, we write our proposed alignment objective which maximizes the social utility as 
%
\begin{align}\label{gen_utility_rlhf_two}
    \pi^*_{\mathcal{F}} \in \arg \max_{\pi} \min_{u \in \mathcal{U}}F_{r_{\phi_u^*}}(\pi) -\beta\mathbb{D}_{\textrm{KL}}\big[\pi|| \pi_{\text{ref}}\big], 
\end{align}
where, $F_{r_{\phi_u^*}}(\pi):=\mathbb{E}_{\mathbf x\sim \mathcal{P},\mathbf{y}\sim \pi(\cdot~|~\mathbf{x}) }\big[{r_{\phi_u^*}}(\mathbf y, \mathbf x)\big]$ (cf. \eqref{eq:RL}) represents the alignment objective for the $u^{\texttt{th}}$ sub-population or group among set of humans.

\begin{algorithm}[t]
\caption{MaxMin RLHF}
\label{alg:cdm2}
\begin{algorithmic}[1]
\STATE {\bf Input}: Preference dataset $\mathcal{D}$, initial reward parametrization for each subpopulation $u$ as $r_{\phi_0}^u$, initial policy parameter $\pi_0$.
\STATE {\bf Reward Learning with EM}: Utilize Algorithm \ref{alg:cdm} for learning rewards with EM to learn $r_{\phi}^u$ for all user subpopulation $u$ 
\STATE \textbf{Max-Min Policy Iteration}:
\FOR{$t = 0$ to $T-1$}
    \STATE {\bf Choosing Minimum Utility Subpopulation}:
    \STATE \( u_{\text{min}} \gets \arg\min_{\mathcal{H}_u \in \mathcal{U}} F_{r_{\phi}^u}(\pi_t) \)
    \STATE {\bf Perform the PPO Update}:
    \STATE Update policy $\pi$ towards maximizing the objective:
    \STATE \( \pi_{i+1} \gets \text{PPO-update}(F_{r_{\phi_u^*}}(\pi_t) -\beta\mathbb{D}_{\textrm{KL}}\big[\pi_t|| \pi_{\text{ref}}\big]) \)
\ENDFOR
\STATE {\bf Output}: Policy $\pi_T$ aligned with socially fair preference dataset
\end{algorithmic}
\end{algorithm}

\textbf{MaxMin RLHF.}  If we have access to individual human sub-population rewards, we can go directly to solve the optimization problem in \eqref{gen_utility_rlhf_two} with the algorithm summarized in Algorithm \ref{alg:cdm2}. But often, in practice, they are hardly available. To address this challenge,
we consider an expectation-maximization algorithm to learn a mixture of reward models summarized in Algorithm \ref{alg:cdm} which learns the $r_{\phi_u}$'s and the $|\mathcal{U}|$ clusters.

We summarize the EM algorithm for reward learning in Algorithm \ref{alg:cdm}.

\begin{algorithm}[t]
\caption{Learning Rewards with EM Algorithm}
\label{alg:cdm}
\begin{algorithmic}[1]
\STATE {\bf Input}: Preference data $\mathcal{D}$, $|\mathcal{U}|$ clusters of users among all humans in $\mathcal{H} = \bigcup_{u =1}^{|\mathcal{U}|} \mathcal{H}_u$, pretrained $\{r_{\phi_u}\}_{u=1}^{|\mathcal{U}|}$, loss function $\operatorname{loss}$, convergence criteria\\

\WHILE{not reach the convergence criteria}
    \FOR{$h \in \mathcal{H}$}
        \STATE {\bf E-step (hard cluster assignment)}: 
        assign $h$ to the $u$-th cluster s.t.
        $$
            u =\ arg\max_{u \in {1,\cdots,|\mathcal{U}|}} \prod_{ (\mathbf{x}, \mathbf{y_1}, \mathbf{y_2}, h) \in \mathcal{D}} w(\phi_u, \mathbf{x}, \mathbf{y_1}, \mathbf{y_2})
        $$
        where $w(\cdot)=\frac{\exp\left(r_{\phi_u}(\mathbf{y_1}, \mathbf{x})\right)}{\exp\left(r_{\phi_u}(\mathbf{y_1}, \mathbf{x})\right) + \exp\left(r_{\phi_u}(\mathbf{y_2}, \mathbf{x})\right)}$
    \ENDFOR
    
    \STATE {\bf M-step}: Update each $\phi_u, u = 1,\cdots,|\mathcal{U}|$ by minimizing the negative log-likelihood loss (\ref{eq:reward_model}) on the assigned users' data
    
\ENDWHILE   
\end{algorithmic}
\end{algorithm}

\section{Experimental Results}\label{experiments}
In this section, we present a comprehensive empirical evaluation of the alignment impossibilities and our proposed solutions for language models, structured into two distinct subsections: \emph{Small Scale} experiments (Sec. \ref{small_Scale}) for initial proof of concept, and \emph{Large Scale} experiments (Sec. \ref{large_Scale}) for broader validation. We first demonstrate the practical challenges of alignment (cf. Theorem \ref{theorem_2_impossibility}), followed by showcasing the efficacy of our MaxMin-RLHF strategy. This approach illustrates that, with a focus on social welfare objectives, alignment across diverse human preferences is attainable.

\subsection{Small Scale Experiments (with GPT-2): Sentiment and Conciseness Alignment}\label{small_Scale}
\begin{figure}[t]
    \hspace{5mm}\includegraphics[width=0.9\columnwidth]{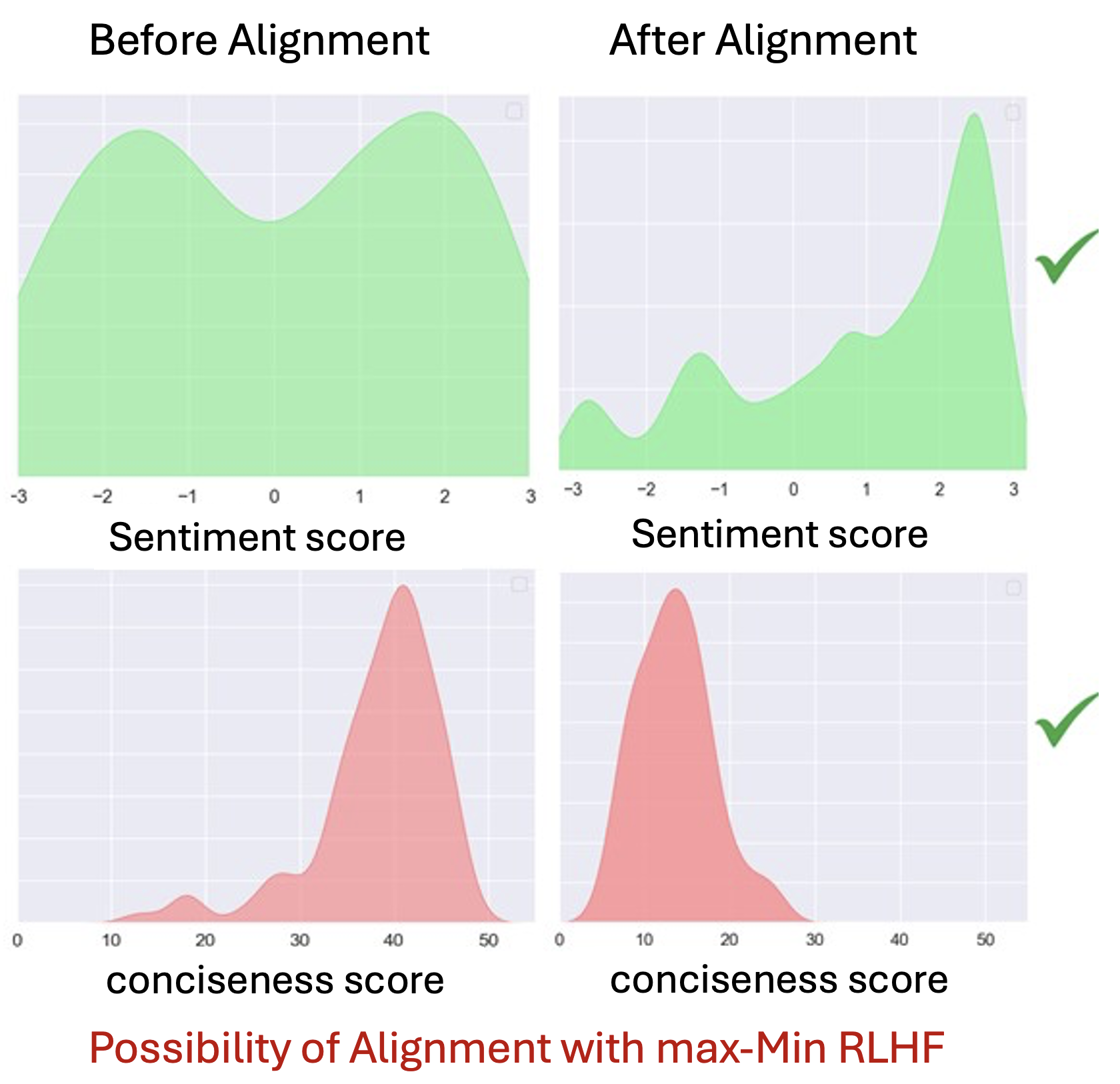}
\vspace{-1mm}
    \caption{\textbf{(Alignment with MaxMin RLHF).} This figure shows the performance of our proposed MaxMin RLHF algorithm for the preference dataset described in Figure \ref{fig:diversity}. The task is to align a language model to generate positive sentiment responses that are concise (of shorter token length) in nature. We note that MaxMin-RLHF aligned language model can generate highly positive sentiment sentences and satisfy the conciseness criteria. This shows alignment with both the majority and minority preferences.}
    \label{fig:possi}
\end{figure}

\textbf{Dataset.} For the experiment in this section on controlled sentiment generation, we categorized the humans into two groups: \emph{majority} (Group 1) and \emph{minority} (Group 2). In these sub-groups,  Group 1 prefers responses with positive sentiment, and Group 2 prefers brevity (conciseness) in responses. 
 We use the IMDb dataset as a basis for our inputs \citep{maas-EtAl:2011:ACL-HLT2011}, the goal for the optimal policy is to produce responses $\textbf{y}$ that exhibit positive sentiment (catering to Group 1) while remaining concise (catering to Group 2). We generated two sets of preference pairs for a controlled evaluation for each user group. For Group 1, we utilized a pre-trained sentiment classifier to ensure $p(\text{positive}\mid \textbf{x},\textbf{y}_1)>p(\text{positive}\mid \textbf{x},\textbf{y}_2)$ and similarly for Group 2 we preferred shorter responses over longer ones. To illustrate the majority and minority group dynamics, we control the proportion of the user groups in the preference data (Group 1: 80\% and Group 2 - 20\%).  For the experiments in this subsection, we use GPT-2 \citep{Radford2019LanguageMA} as the base model.

\textbf{Impossibility Results.} To demonstrate our impossibility results as stated in Theorem \ref{theorem_2_impossibility}, we perform the three steps of RLHF (described in \citep{christian2020alignment, ouyang2022training}) as prevalent currently with a single utility reward function on the combined preference dataset. For SFT, we fine-tune GPT-2 until convergence on reviews from the train split of the IMDB dataset and use this GPT-2 backbone for both the reward model and PPO training. The generations are evaluated against the ground truth rewards $r_1^*$ for positive sentiment (majority group) and $r_2^*$ for conciseness (minority group). It is evident from Figure \ref{fig:impossibility} that the generated responses are significantly biased toward the majority user group's preference who preference positive sentiment (note high sentiment score (green curve, high score is better) after alignment) while the preferences (concise responses) of the minority user group were neglected (note high conciseness score (red curve, lower score is better) after alignment), resulting in more verbose generations than desired. 

\textbf{Proposed MaxMin RLHF.} Our proposed algorithm can efficiently align to both group preferences as shown in Figure \ref{fig:possi} thereby generating responses that are of positive sentiment and concise and thus cater to both the majority and minority user groups mitigating the social disparity. We further collectively present the average performance of MaxMin RLHF with the single reward RLHF and baseline model in Figure \ref{fig:enter-label}.
\begin{figure}[h]
    \centering
    \includegraphics[width=\columnwidth]{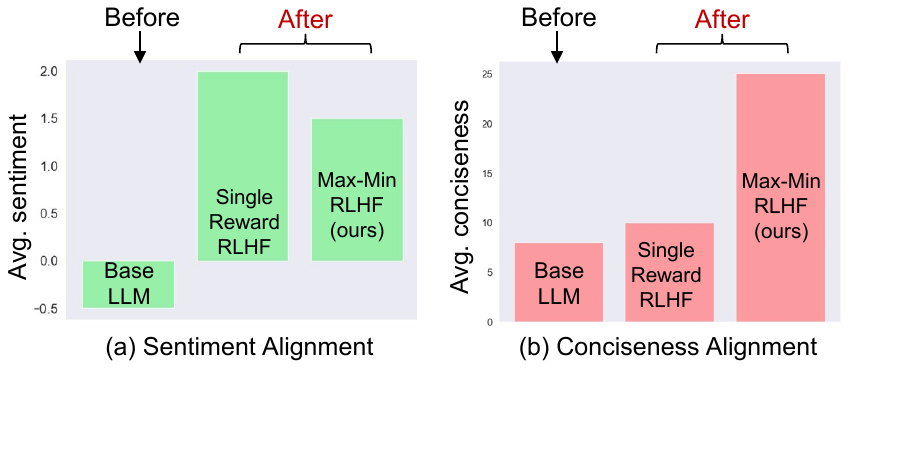}\vspace{-3mm}
    \caption{This figure shows the average performance in terms of sentiments of the generated output and the conciseness alignment.We note that MaxMin RLHF is able to better cater to both the alignment criteria as compared to single reward RLHF as expected.}
    \label{fig:enter-label}
\end{figure}
\subsection{Large Scale Experments (with Tulu2-7B)}\label{large_Scale}

\textbf{Datasets and Experimental Setup.} We use the same dataset as \citet{jang2023personalized} and $10$k data points from GPT4-Alpaca \citep{peng2023instruction} are used as the instruction dataset to generate rollouts, collect pairwise feedback data, and PPO training. We utilize GPT-4 to simulate human annotators with preference prompts described in Table \ref{dataset_summary} in Appendix \ref{appendix_experiments}. We divide the datasets into groups of human users. Each group has $40$ users, which are split into $30$ users in training data and $10$ users in testing data. For the experiments in this subsection, we use Tulu2-7B \citep{ivison2023camels} as the base model. For each dataset, P1, P2, and P3, we mix the training user groups to build the simulation dataset. We have 60 users in training data which are mixed from two different groups with diverse preferences. The original distribution is that users are evenly distributed in two clusters. Then, we use the EM algorithm to train $|\mathcal{U}|=2$ reward models until we converge. Update $\phi_u, u = 1,\cdots,|\mathcal{U}|$ by minimizing the negative log-likelihood loss (\ref{eq:reward_model}). Then, trained model is used to assign clusters to users in testing data.

\subsubsection{Main Results}

\textbf{Impossibility of Single Reward Model.} When the user groups are biased (divided into majority and minority groups based on the preference dataset), the single reward model fails to capture the preferences of minority user groups. We test on preference dataset P1A/P1B representing two user groups and adjust the ratio of the number of users from group P1A and group P1B. Table \ref{tab:table2} summarizes the accuracy for the majority group and minority group, as well as the accuracy on the total data.  Here, low accuracy means that the alignment with the minority user group will be poor after the PPO step since the reward model itself is not accurate. 
\begin{table}[t]
    \centering
     \resizebox{0.7\columnwidth}{!}{\begin{tabular}{cccc}
    \toprule
        Ratio & Total& Majority & Minority\\
    \midrule
        1:1 & 0.686 & 0.668 & 0.704\\
        2:1 & 0.608 & 0.728 & 0.488\\
        6:1 & 0.588 & 0.724 & 0.452\\
        10:1 & 0.568 & 0.716 & 0.42\\
        \hline
    \end{tabular}}
    \caption{This table presents the test accuracy of the single reward model training on the preference dataset and shows its failure to align with the minority. The first column denotes the user group ratio in the dataset, the second column shows the total accuracy, the third column shows the accuracy of the majority group, and the fourth column shows the accuracy of the minority group. }
    \label{tab:table2}
\end{table}
%

\textbf{Reward Learning with EM (Algorithm \ref{alg:cdm}).} Following the procedures in the experiment setup, we get similar and good results on all three datasets, as shown in Figure \ref{fig:result1}. From the results in Figure \ref{fig:result1}, we note that after the fourth iteration, all users are clustered correctly, meaning the mixture preference model successfully converges we successfully learn diverse groups of users with diverse preferences.

\textbf{MaxMin RLHF Alignment.} We further test the performance of our MaxMin-RLHF alignment method and compare it with the single reward \begin{table}[h]
    \centering
    \resizebox{0.7\columnwidth}{!}{\begin{tabular}{cccc}
    \toprule
        Method & P3A & P3B & Average\\
    \midrule
        MaxMin & 57.78 & 55.56 & 56.67\\
        1:1 & 55.85 & 52.62 & 54.24\\
        2:1 & 55.56 & 48.89 & 52.23\\
        6:1 & 58.06 & 46.67 & 52.37\\
        10:1 & 56.00 & 45.00 & 50.50\\
        \hline
    \end{tabular}}
    \caption{Pairwise win rate (\%) on P3 dataset using GPT-4.}
    \label{tab:tableP3}
\end{table}RLHF models trained on biased datasets. Our baselines include ratios of  1, 2, 6, and 10, the same setting as discussed for Table \ref{tab:table2}. Following \citet{jang2023personalized}, we use the same $50$ instances from Koala evaluation\citep{geng2023koala} and test the model's ability to generate answers in different groups of users' preferences. We run pairwise evaluations by GPT-4 using AlpacaFarm codebase\citep{dubois2023alpacafarm} and use the win rate to the base Tulu2-7B model as the metric. Our results in Table \ref{tab:tableP3} and Table \ref{tab:tableP1} show that MaxMin alignment keeps a high win rate while the models trained by PPO with a single reward model on biased datasets will have a relatively poor performance on the minority data representing minority user groups.
\begin{table}
    \centering
     \resizebox{0.48\columnwidth}{!}{\begin{tabular}{cccc}
    \toprule
        Method & P1A & P1B \\
    \midrule
        MaxMin & 57.50 & 60.00 \\
        1:1 & 56.00 & 51.97 \\
        2:1 & 57.78 & 44.00 \\
        6:1 & 54.81 & 48.00 \\
        10:1 & 55.11 & 45.08 \\
        \hline
    \end{tabular}}
    \hspace{2mm}\resizebox{0.48\columnwidth}{!}{ \begin{tabular}{cccc}
    \toprule
        Method & P2A & P2B \\
    \midrule
        MaxMin & 54.50 & 56.00 \\
        1:1 & 53.73 & 54.00 \\
        2:1 & 55.55 & 51.72 \\
        6:1 & 52.14 & 49.40 \\
        10:1 & 53.96 & 45.98 \\
        \hline
    \end{tabular}}
    \caption{Pairwise winrate (\%) on P1-P2 using GPT-4.}
    \label{tab:tableP1}
\end{table}

\section{Conclusions}
In this work, we critically examine the limitations of the single-reward RLHF framework, particularly its insufficiency in addressing the diversity of human preferences, leading to an impossibility result for alignment with diverse preferences. To achieve a socially fair alignment in diverse human preference settings, we introduce a novel approach called MaxMin-RLHF, which learns a max-min policy over a distribution of reward functions to achieve a more equitable model alignment. Our experiments demonstrate the effectiveness of MaxMin-RLHF in producing socially fairer outcomes, highlighting the need for more inclusive strategies in RLHF methodologies.

\section*{Impact Statement}
The primary objective of this work is to highlight the limitations of the existing alignment techniques in representing diverse opinions and preferences. Our research is one of the first to formally highlight the above limitation with mathematical and empirical demonstrations. Finally, our research demonstrates a first step towards equitable alignment with diverse preferences. The findings of our work would encourage and foster further research in the domain of alignment under diversity, ensuring the current AI models is not biased towards specific minority groups.

\section*{Acknowledgements}
Chakraborty and Huang are supported by DARPA Transfer from Imprecise and Abstract Models to Autonomous Technologies (TIAMAT) 80321, National Science Foundation NSF-IIS-2147276 FAI, DOD-ONR-Office of Naval Research under award number N00014-22-1-2335, DOD-AFOSR-Air Force Office of Scientific Research under award number FA9550-23-1-0048, DOD-DARPA-Defense Advanced Research Projects Agency Guaranteeing AI Robustness against Deception (GARD) HR00112020007, Adobe, Capital One and JP Morgan faculty fellowships. Manocha and Bedi are supported by Army Cooperative Agreement W911NF2120076. Mengdi Wang acknowledges the support by NSF IIS-2107304, NSF CPS-2312093, ONR 1006977 and Genmab. {We also thank Rui Yang and Han Zhao for pointing out a bug in the proof of Lemma 1 in the previous version.}

\bibliography{ref}
\bibliographystyle{icml2024}

\clearpage
\onecolumn
\tableofcontents
\clearpage

\appendix
\section{Notations} 

We define the various notations in this table first.  

\begin{center}
	\resizebox{!}{0.15\textwidth}{\begin{tabular}{ |c|c| } 
		\hline
		Notations & Description \\ 
		\hline
				$\textbf{x}$ & prompt    \\ 
				
		\hline
				$\mathcal{X}$ & set of prompts    \\ 
		\hline
				$\textbf{y}$ & output text generated by the LLM    \\ 
		\hline
				$\pi_{\text{ref}}$ & direct supervised fine-tuning model, takes $\textbf{x}$ as input and generates $\textbf{y}$ as output   \\ 
		\hline
				$(\textbf{y}_1,\textbf{y}_2)$ & output pair generated by LLM   \\ 
		\hline
				$h$ & human  \\ 
		\hline
				$\mathcal{D}$ & dataset which has the data of the form $(\textbf{x},\textbf{y}_1,\textbf{y}_2)$    \\ 
						\hline
				$\phi$ & reward model parameter    \\ 
    		\hline
				$\theta$ & language model parameter    \\ 
						\hline	
				$\mathcal{H}$ & set of human population  \\ 
		\hline
	\end{tabular}}
\end{center}

\section{A detailed Context of Related Works}\label{related_works_appendix}

\textbf{Reinforcement Learning from Human Feedback.} RL methods, such as policy gradient, applied to train language models for long-form generation \cite{cho2018coherent}. 
The current RLHF approaches \citep{rlhf_p1, rlhf_p2, zhu2023principled} involve training a reward model based on human preference feedback and then fine-tuning the language model using proximal policy optimization (PPO) \citep{schulman2017proximal}. The PPO algorithm helps to learn a model that produces responses that maximize the reward \citep{ouyang2022training, bai2022training}. Besides PPO, DPO (Direct Preference Optimization, \citet{rafailov2023direct}) directly trains the large language model using human preferences without training the reward model. A self-play-based approach such as SPIN \citep{chen2024self} is similar to DPO but has an iterative framework. However, most of the existing alignment approaches only consider the average preference by human annotators and ignore the inherent diversity among human preferences \citep{casper2023open, rlhf_survey2}. 
A number of theoretical studies have analyzed the efficiency and benefits for reinforcement learning using preference data \citep{ji2023provable,zhang2023unified,li2023reinforcement}. \citep{chakraborty2024parl} proposed a bilevel reinforcement learning framework for policy alignment.
Recently \citep{santurkar2023whose} created a dataset for evaluating the alignment of language models with $60$ US demographic groups over a wide range of topics and found substantial misalignment between a  selanguage models and those groups. It emphasizes the criticality of considering diversity while performing alignment. 

\textbf{Diversity in Human Preferences.} Here, we briefly review the literature highlighting the reasons for diversity in the context of LLMs. Diverse human preferences stem significantly from various factors related to social and cultural backgrounds \citep{aroyo2023dices, aroyo2023reasonable, denton2021ground}.  The key factors contributing to this diversity include (i) \textit{socio-demographic backgrounds}, including race, ethnicity, age, and gender shape preferences. Gender differences, for example, influence sensitivity to online content, with women facing more online harassment \cite{vogels2021state}. {(ii) \textit{Personal bias and context subjectivity}}, which affects the human preferences for controversial topics in interpreting language and divisive themes \citep{diversity3, diversity2}). {(iii) \textit{Imperfect preferences},} which arises due to variations in expertise, training, or quality control leading to diverse preferences, with certain content inaccurately considered offensive by some groups \citep{diversity2}. {(iii) \textit{Linguistic ambiguity \& missing context},} could lead to diversity because of words or phrases with multiple possible interpretations and without clear context \citep{diversity2, diversity3, diversity4}. These factors collectively underscore the complexity of aligning LLM outputs with the diverse preferences of human users, demonstrating the importance of recognizing and addressing the multifaceted nature of user feedback.

\section{Preliminary Results}
 We present the following preliminary results in the form of Lemma \ref{lemma_2} and Lemma \ref{lemma_1}. 

\begin{lem}\label{lemma_2}
\vspace{2mm}
    The parametrized preference probability distribution $p_{\phi}(\mathbf{y}_1 \succ \mathbf{y}_2 \mid \mathbf{x}) = 
\frac{\exp(r_{\phi}(\mathbf{y}_1, \mathbf{x}))}
{\exp(r_{\phi}(\mathbf{y}_1, \mathbf{x})) + \exp(r_{\phi}(\mathbf{y}_2, \mathbf{x}))}$ under the Bradley -Terry model \citep{bradley1952rankanalysis} is Lipschitz with respect to parameter $\phi$. This implies that
\begin{align}
    |p_{\phi}(\mathbf{z})-p_{\phi'}(\mathbf{z})|\leq L_{p}\|\phi-\phi'\|, 
\end{align}
where $\mathbf z:=(\mathbf y_1 \succ \mathbf y_2 \mid \mathbf x)$, $L_p=4D$, and $D$ denotes the upper bound on the feature representation $\|\psi(\bby,\bbx)\|\leq D$ for all $(\bbx,\bby)$. 
\end{lem}
\begin{proof}
    Let us start from the definition of $p_{\phi}(\mathbf{y}_1 \succ \mathbf{y}_2 \mid \mathbf{x})$ given by
\begin{align} \label{model}
p_{\phi}(\mathbf{y}_1 \succ \mathbf{y}_2 \mid \mathbf{x}) &= 
\frac{\exp(r_{\phi}(\mathbf{y}_1, \mathbf{x}))}
{\exp(r_{\phi}(\mathbf{y}_1, \mathbf{x})) + \exp(r_{\phi}(\mathbf{y}_2, \mathbf{x}))} = \frac{1}{1 + \exp(-(r_{\phi} ({\mathbf{y}_1}, \mathbf{x}) - (r_{\phi} ({\mathbf{y}_2}, \mathbf{x}))}. %
\end{align}
From the definition of the Bradley-Terry preference model from equation \eqref{eq:bradley-terry} with the linear parametrization of the reward function as $r_{\phi} ({\mathbf{y}}, \mathbf{x}) = \langle\phi, \psi({\mathbf{y}}, \mathbf{x})\rangle$, we can write the  equality in \eqref{model} as
\begin{align}\label{model_2}
p_{\phi}(\mathbf{y}_1 \succ \mathbf{y}_2 \mid \mathbf{x}) 
&= \frac{1}{1 + \exp(-(\langle\phi, \psi({\mathbf{y}_1}, \mathbf{x})\rangle - \langle\phi, \psi({\mathbf{y}_2}, \mathbf{x})\rangle)))} 
\nonumber
\\
&= \frac{1}{1 + \exp(-\langle\phi, \psi'(\mathbf{y}_1, \mathbf{y}_2,\mathbf{x})\rangle)},
\end{align}
where we define $\psi'(\mathbf{y}_1, \mathbf{y}_2,\mathbf{x}):=\psi({\mathbf{y}_1}, \mathbf{x})- \psi({\mathbf{y}_2}, \mathbf{x})\rangle$ for the ease of notation. Next, 
 differentiating both sides in \eqref{model_2} with respect to $\phi$, we obtain
\begin{align}
\nabla_{\phi} p_{\phi}(\mathbf{y}_1 \succ \mathbf{y}_2 \mid \mathbf{x}) &= 
- \psi'(\mathbf{y}_1, \mathbf{y}_2,\mathbf{x}) \cdot \frac{\exp(-\langle\phi,\psi'(\mathbf{y}_1, \mathbf{y}_2,\mathbf{x})\rangle)}{(1 + \exp(-\langle\phi,\psi'(\mathbf{y}_1, \mathbf{y}_2,\mathbf{x})\rangle))^2} \nonumber\\
&= - \psi'(\mathbf{y}_1, \mathbf{y}_2,\mathbf{x}) \left[ \frac{1}{1 + \exp(-\langle\phi,\psi'(\mathbf{y}_1, \mathbf{y}_2,\mathbf{x})\rangle)} - \frac{1}{(1 + \exp(-\langle\phi,\psi'(\mathbf{y}_1, \mathbf{y}_2,\mathbf{x})\rangle))^2} \right] .
\end{align}
Taking the norm on both sides and applying Cauchy-Schwartz inequality, we get
\begin{align}
\| \nabla_{\phi} p_{\phi}(\mathbf{y}_1 \succ \mathbf{y}_2 \mid \mathbf{x}) \| &\leq  
\| \psi'(\mathbf{y}_1, \mathbf{y}_2,\mathbf{x}) \| \left[ \frac{1}{1 + \exp(-\langle\phi,\psi'(\mathbf{y}_1, \mathbf{y}_2,\mathbf{x})\rangle)} + \frac{1}{(1 + \exp(-\langle\phi,\psi'(\mathbf{y}_1, \mathbf{y}_2,\mathbf{x})\rangle))^2} \right]
\nonumber
\\&\leq 2 \| \psi'(\mathbf{y}_1, \mathbf{y}_2,\mathbf{x}) \|.
\end{align}
From the definition of $\psi'(\mathbf{y}_1, \mathbf{y}_2,\mathbf{x})$ and the boundedness of the feature representations, we note that $\|\psi'(\mathbf{y}_1, \mathbf{y}_2,\mathbf{x})\|=\|\psi({\mathbf{y}_1}, \mathbf{x})- \psi({\mathbf{y}_2}, \mathbf{x})\rangle\|\leq 2D$. Hence, we obtain the final bound  
\begin{align}
\| \nabla_{\phi} p_{\phi}(\mathbf{y}_1 \succ \mathbf{y}_2 \mid \mathbf{x}) \| &\leq 4 D.
\end{align}
 Hence proved.

\end{proof}

\begin{lem}\label{lemma_1}
\vspace{2mm}
The cross-entropy loss minimization for reward learning in step 2 in the RLHF pipeline (cf. \eqref{eq:reward_model}) leads to implicit weightage minimization among the user groups. Specifically, the loss function minimizes the distance to distribution $p_{\phi^*}(\bbz) = \sum_{u=1}^{ |\mathcal{U}|} \eta(u)p_u^* (\bbz)$, where $\eta$ is the implicit distribution among user groups. 
\end{lem}
\begin{proof}[Proof of Lemma \ref{lemma_1}]
From the equality in \eqref{new_definition}, we note that we can write
$p^*(\mathbf y_1 \succ \mathbf y_2 \mid \mathbf x)=\mathbb{E}_{u} [p^*_u(\mathbf y_1 \succ \mathbf y_2 \mid \mathbf x)]$. With this notation, the loss function for reward learning in \eqref{loss} can be written as 
\begin{align}\label{mle_new_22}
    \mathcal{L}_R(r_{\phi}, \mathcal{D})&= - \mathbb{E}_{(\mathbf x, \mathbf y_1, \mathbf y_2)\sim \mathcal{D}}\Big[ \mathbb{E}_{u} [p^*_u(\mathbf y_1 \succ \mathbf y_2 \mid \mathbf x)] \log p_{\phi}(\succ) + \mathbb{E}_{u} [p^*_u(\mathbf y_1 \prec \mathbf y_2 \mid \mathbf x)] \log p_{\phi}(\prec) \Big],
\end{align}
where the equation incorporates the individual user group's optimal $p_u^*$ (we denote the corresponding individual optimal reward parameter by $\phi_u^*$) in the likelihood objective.  
 As a first step, let us decompose \eqref{mle_new_22} as
\begin{align}\label{final}
   \mathcal{L}_R&(r_{\phi}, \mathcal{D})\nonumber\\
     =&  \mathbb{E}_{(\mathbf x, \mathbf y_1, \mathbf y_2)}\Bigg[\sum_{u=1}^{ |\mathcal{U}|} [p_u^* (\bby_1\succ\bby_2|\bbx) \eta(u)] \log p_{\phi}(\succ) - \sum_{u=1}^{ |\mathcal{U}|} [p_u^* (\bby_1\succ\bby_2|\bbx) \eta(u) \log p_u^*(\bby_1\succ\bby_2|\bbx)]  \nonumber\\
    & \hspace{16mm}+ \sum_{u=1}^{ |\mathcal{U}|} [p_u^* (\bby_1\prec\bby_2|\bbx) \eta(u)] \log p_{\phi}(\prec) - \sum_{u=1}^{ |\mathcal{U}|} [p_u^* (\bby_1\prec\bby_2|\bbx) \eta(u) \log p_u^*(\bby_1\succ\bby_2|\bbx)]  \nonumber\\ 
    & \hspace{17mm} + \sum_{u=1}^{ |\mathcal{U}|} p_u^* (\bby_1\succ\bby_2|\bbx) \eta(u) \log p_u^*(\bby_1\succ\bby_2|\bbx) + \sum_{u=1}^{ |\mathcal{U}|} p_u^* (\bby_1\prec\bby_2|\bbx) \eta(u) \log p_u^*(\bby_1\prec\bby_2|\bbx)\Bigg],
\end{align}
where, we add and subtract $\sum_{u=1}^{ |\mathcal{U}|} [p_u^* (\bby_1\succ\bby_2|\bbx) \eta(u) \log p_u^*(\bby_1\succ\bby_2|\bbx)]$ and $\sum_{u=1}^{ |\mathcal{U}|} [p_u^* (\bby_1\prec\bby_2|\bbx) \eta(u) \log p_u^*(\bby_1\prec\bby_2|\bbx)]$ to get the final expression. After rearranging the terms in \eqref{final}, we get
\begin{align}\label{mle_impl}
   \mathcal{L}_R(r_{\phi}, \mathcal{D})
     =& - \mathbb{E}_{(\mathbf x, \mathbf y_1, \mathbf y_2)\sim \mathcal{D}}\Bigg[\sum_{u=1}^{ |\mathcal{U}|} [p_u^* (\bby_1\succ\bby_2|\bbx) \eta(u)] \Big(\log p_{\phi}(\succ) - \log p_u^*(\bby_1\succ\bby_2|\bbx)\Big) \\ \nonumber
    & \hspace{25mm}+ \sum_{u=1}^{ |\mathcal{U}|} [p_u^* (\bby_1\prec\bby_2|\bbx) \eta(u)] \Big(\log p_{\phi}(\prec) - \log p_u^*(\bby_1\prec\bby_2|\bbx)\Big) \\ \nonumber
    & \hspace{30mm} + \sum_{u=1}^{ |\mathcal{U}|} p_u^* (\bby_1\succ\bby_2|\bbx) \eta(u) \log p_u^*(\bby_1\succ\bby_2|\bbx) \\ \nonumber
    & \hspace{35mm}+ \sum_{u=1}^{ |\mathcal{U}|} p_u^* (\bby_1\prec\bby_2|\bbx) \eta(u) \log p_u^*(\bby_1\prec\bby_2|\bbx)\Bigg] \\ \nonumber
    & =  - \mathbb{E}_{x , y_1, y_2}\Bigg[\sum_{u=1}^{ |\mathcal{U}|} \eta(u) \Big(p_u^* (\bby_1\succ\bby_2|\bbx) \cdot\log \frac{p_{\phi}(\bby_1\succ\bby_2|\bbx)}{p_u^*(\bby_1\succ\bby_2|\bbx)}+ p_u^* (\bby_1\prec\bby_2|\bbx)\cdot\log \frac{p_{\phi}(\bby_1\prec\bby_2|\bbx)}{p_u^*(\bby_1\prec\bby_2|\bbx)}\Big) \\ \nonumber
    &\ \ \  + \sum_{u=1}^{ |\mathcal{U}|} p_u^* (\bby_1\succ\bby_2|\bbx) \eta(u) \log p_u^*(\bby_1\succ\bby_2|\bbx) + \sum_{u=1}^{ |\mathcal{U}|} p_u^* (\bby_1\prec\bby_2|\bbx) \eta(u) \log p_u^*(\bby_1\prec\bby_2|\bbx)\Bigg]. 
\end{align}
Next, by utilizing the definition of KL-divergence and entropy to get the final expression as follows 
\begin{align}\label{final_2oss}
     \mathcal{L}_R(r_{\phi}, \mathcal{D})& = \mathbb{E}_{(\mathbf x, \mathbf y_1, \mathbf y_2)\sim \mathcal{D}} \Bigg[\sum_{u=1}^{ |\mathcal{U}|} \eta(u) \mathsf{KL} (p_u^* || p_{\phi}) + \eta(u) \mathsf{H}(p_u^*)\Bigg].
\end{align}
From the above objective in \eqref{final_2oss}, we note that the objective is minimized for $\phi$ when $\sum_{u=1}^{ |\mathcal{U}|} \eta(u) \mathsf{KL} (p_u^* || p_{\phi})= 0.$
To proceed further, let us focus on the term $\sum_{u=1}^{ |\mathcal{U}|} \eta(u) \mathsf{KL} (p_u^* || p_{\phi})$ from equation \eqref{final_2oss} as
\begin{align} \label{kl_exp12222}
    \sum_{u=1}^{ |\mathcal{U}|} \eta(u) \mathsf{KL} (p_u^* || p_{\phi})& = \sum_{u=1}^{ |\mathcal{U}|} \eta(u) \sum_\bbz  p_u^*(\bbz) \log \frac{p_u^*(\bbz)}{p_{\phi}(\bbz)} \nonumber\\ 
    & = \sum_{u=1}^{ |\mathcal{U}|} \eta(u) \sum_\bbz  p_u^* (\bbz) \log p_u^*(\bbz) - \sum_{u=1}^{ |\mathcal{U}|} \eta(u) \sum_\bbz  p_u^* (\bbz) \log p_{\phi}(\bbz) \nonumber\\ 
    & = -\sum_{u=1}^{ |\mathcal{U}|} \eta(u) H(p_u^*) - \sum_\bbz \log p_{\phi}(\bbz) \underbrace{\sum_{u=1}^{ |\mathcal{U}|} \eta(u)  p_u^* (\bbz)}_{=p^* (\bbz)}.  
    \end{align}
    From the definition of KL d in \eqref{new_definition}, it holds that
\begin{align} \label{kl_exp122}
    \sum_{u=1}^{ |\mathcal{U}|} \eta(u) \mathsf{KL} (p_u^* || p_{\phi})& = -\sum_{u=1}^{ |\mathcal{U}|} \eta(u) H(p_u^*) - \sum_\bbz p^* (\bbz)\log p_{\phi}(\bbz)  .  
    \end{align}
    Next, by adding and subtracting the term $\sum_\bbz p^* (\bbz)\log p^* (\bbz)$ in the right hand side of \eqref{kl_exp122}, we get
    \begin{align}
   \sum_{u=1}^{ |\mathcal{U}|} \eta(u) \mathsf{KL} (p_u^* || p_{\phi})& = -\sum_{u=1}^{ |\mathcal{U}|} \eta(u) H(p_u^*) - \sum_\bbz p^* (\bbz)\log p_{\phi}(\bbz)   + \sum_\bbz p^* (\bbz)\log p^* (\bbz) - \sum_\bbz p^* (\bbz)\log p^* (\bbz) \\ \nonumber
    & = - H (p^*) - \sum_{u=1}^{ |\mathcal{U}|} \eta(u) H(p_u^*) + \mathsf{KL} (p^* || p_{\phi}).
\end{align}
Now, replacing this expression in the original implicit minimization objective in  \eqref{final_2oss}, we note that the minimization will be achieved when $p_{\phi^*}(\bbz) = \sum_{u=1}^{ |\mathcal{U}|} \eta(u)  p_u^* (\bbz)$ for all $z$. Hence, the reward learning objective is implicitly learning a weighted combination, which would lead to a significant gap in individual utilities, as discussed in the subsequent section.
\end{proof}

\section{ Proof of Lemma \ref{theorem_1_impossibility0}}\label{theorem_1_impossibility_proof}

\begin{proof}
Let us reconsider the reward learning loss $\mathcal{L}_R(r_{\phi}, \mathcal{D})$ whose empirical version is minimized to obtain parameter $\hat{\phi}_{\texttt{MLE}}$ which is the approximation to the true parameter $$\phi^* := \arg \min_{\phi} - \mathbb{E}\bigg[\sum_\bbz p_{\phi^*}(\bbz) \log p_{\phi}(\bbz)\bigg]. $$ As discussed in Sec. \ref{reward_2earning}, due to the presence of diverse human user groups, a $\phi^*_u$ which, is user group specific, will also exist. Our goal is to characterize the gap between $\hat{\phi}_{\texttt{MLE}}$ and $\phi^*_u$ defined as
\begin{align}\label{new}
    \Delta_u^r := \hat{\phi}_{\texttt{MLE}} - \phi_u^*, 
\end{align}
where the optimal $\phi_u^*$ for the user group $u$ is given by
\begin{align}\label{ind_opt}
    \phi_u^* := \arg \min_{\phi} - \mathbb{E}\bigg[\sum_\bbz p_u^*(\bbz) \log p_{\phi}(\bbz)\bigg].
\end{align}
Let us consider the idealistic setting of infinite data under which we know that  $\texttt{MLE}$ would converge to optimal $\phi^*$ \citep{zhu2023principled}. Hence, to proceed further, let us add subtract $\phi^*$ in the right-hand side of \eqref{new}, we get 
\begin{align}\label{delta_net}
    \Delta_u^r = \underbrace{\hat{\phi}_{\texttt{MLE}} - \phi^*}_{=0} +  {\phi^* - \phi_u^*}.
\end{align}
 To derive the lower bound on the reward suboptimality $\Delta_u^r$, we begin with the definition of the total variation distance as 
\begin{align}\label{imp1}
    \text{TV }(p_{\phi_u^*}, p_{\phi^*}) &= \frac{1}{2} \sum_\bbz |p_{\phi_u^*}(\bbz) - p_{\phi^*}(\bbz)|. 
\end{align}
From the Lipschitzness of the preference probability as derived in Lemme \ref{lemma_2}, we can write
\begin{align}\label{imp2}
    \text{TV }(p_{\phi_u^*}, p_{\phi^*}) & \leq 4D\|\phi_u^* - \phi^*\|,
\end{align}
where the multiplication of $2$ comes from the fact that there are two terms in the summation in the right side of \eqref{imp1} (cf. Sec. \ref{reward_2earning}).
From the lower bound in \eqref{imp1} and the expression in \eqref{delta_net}, we obtain
\begin{align}\label{delta_net2}
    \|\Delta_u^r\|\geq \frac{1}{4D}\text{TV }(p_{\phi_u^*}, p_{\phi^*}).
\end{align}
Next, to obtain a lower bound on the term $\text{TV }(p_{\phi_u^*}, p_{\phi^*})$, we begin with the definition of the total variation distance $TV(p_{\phi_u^*}, p_{\phi_j^*}) $ as 
\begin{align}\label{new_proof1}
TV(p_{\phi_u^*}, p_{\phi_j^*}) 
        &= \frac{1}{2} \sum_z |p_{\phi_u^*}(z) - p_{\phi_j^*} (z)| \nonumber\\ 
        & = \frac{1}{2} \sum_z |p_{\phi_u^*}(z) - p_{\phi^*}(z) + p_{\phi^*}(z) - p_{\phi_j^*} (z)| \nonumber\\ 
        & \leq \frac{1}{2} \sum_z |p_{\phi_u^*}(z) - p_{\phi^*}(z)| + \frac{1}{2} \sum_z |p_{\phi^*}(z) - p_{\phi_j^*} (z)|\nonumber \\ 
        & = TV(p_{\phi_u^*}, p_{\phi^*}) + TV(p_{\phi_j^*}, p_{\phi^*}).
\end{align}
In \eqref{new_proof1}, the first equality holds from the definition of TV norm distance, the second equality holds because we add subtract $p_{\phi^*}(z)$ inside the norm, we used triangle inequality for the third inequality, and then again utilize the definition of TV distance to write the final equality. After rearranging the terms in \eqref{new_proof1}, we can write 
\begin{align}\label{new_proof2}
TV(p_{\phi_u^*}, p_{\phi^*}) \geq  TV(p_{\phi_u^*}, p_{\phi_j^*})  - TV(p_{\phi_j^*}, p_{\phi^*}).
\end{align}
From the definition of $p_{\phi^*}(z) =  \sum_{k=1}^{|\mathcal{U}|} n(k) p_{\phi_k^*}(z)$ and from the property of TV distance and Jensen's inequality, we know that $ TV(p_{\phi_j^*}, p_{\phi^*}) \leq \sum_{k=1}^{|\mathcal{U}|} \eta(k) TV(p_{\phi_k^*}, p_{\phi_j^*})$. This would imply that $ -TV(p_{\phi_j^*}, p_{\phi^*}) \geq -\sum_{k=1}^{|\mathcal{U}|} \eta(k) TV(p_{\phi_k^*}, p_{\phi_j^*})$, which we utilize in the right hand side of \eqref{new_proof2} to obtain
\begin{align}\label{new_proof3}
TV(p_{\phi_u^*}, p_{\phi^*}) & \geq  TV(p_{\phi_u^*}, p_{\phi_j^*})  - \sum_{k=1}^{|\mathcal{U}|} \eta(k) TV(p_{\phi_k^*}, p_{\phi_j^*}) \nonumber \\
& = (1 -\eta(u)) TV(p_{\phi_u^*}, p_{\phi_j^*})  - \sum_{k \neq u} \eta(k) TV(p_{\phi_k^*}, p_{\phi_j^*}),
\end{align}
which demonstrates the final expression of our lower bound on the term $TV(p_{\phi_u^*}, p_{\phi^*})$, which holds in general. Considering the second term in the right-hand side of \eqref{new_proof3}, we can write
\begin{align}
    \sum_{k \neq u} \eta(k) TV(p_{\phi_k^*}, p_{\phi_j^*}) \leq TV(p_{\phi^*_{k_{\max}(j)}}, p_{\phi_j^*}) \sum_{k \neq u} \eta(k),
\end{align}
where $TV(p_{\phi^*_{k_{\max}(j)}}, p_{\phi_j^*}):= \max_{k\neq u} TV(p_{\phi^*_{k}}, p_{\phi_j^*})$ for a given $j$. From the definition of weights $\eta$, we know that $ \sum_{k \neq u} \eta(k)=(1-\eta(u))$, hence we can write
\begin{align}\label{upper_new}
    \sum_{k \neq u} \eta(k) TV(p_{\phi_k^*}, p_{\phi_j^*}) \leq (1-\eta(u)) TV(p_{\phi^*_{k_{\max}(j)}}, p_{\phi_j^*}).
\end{align}
Utilizing the upper bound of \eqref{upper_new} into the right hand side of \eqref{new_proof3}, we obtain
\begin{align}\label{new_proof33}
TV(p_{\phi_u^*}, p_{\phi^*}) & \geq  (1 -\eta(u)) \Big[TV(p_{\phi_u^*}, p_{\phi_j^*})-TV(p_{\phi^*_{k_{\max}(j)}}, p_{\phi_j^*})\Big].
\end{align}
In the above expression, we note from the right-hand side that the lower bound in \eqref{new_proof33} holds for all $u$ and $j$. This implies that the right-hand side in \eqref{new_proof33} will be either positive or negative for $(u,j)$ pairs. But an interesting point to note here is that if the right-hand side is lower bounded away from zero, even for one $(u,j)$ pair, that is an impossibility result, for the corresponding $u$ which is the minority user.  To proceed, as a first step, we find the most diverse user-group ($i, j \in \mathcal{U}$) defined as $(i^*, j^*)$ be the pair of users with the maximum total variation distance, given by
\begin{align}
   (i^*, j^*) = \arg\max_{i, j \in \mathcal{U}} TV(p_{\phi_i}, p_{\phi_j}).
\end{align}
We next define the minority user group $u^*$ as the group with the largest total variation distance from the remaining user groups $k \in \mathcal{U}, k \neq j^*$, where we quantify this by computing the maximum distance to any other user-groups in the set $\mathcal{U}$. Assuming there exists an 
\begin{align}
   u^* = \arg\max\bigg\{\max_{k\neq j} \text{TV}(p^*_{i^*},p^*_k), \max_{k\neq i} \text{TV}(p^*_{j^*},p^*_k)\bigg\}.
\end{align}
Thus in this process, we select that user-group which is at maximum distance from the set $\mathcal{U}$ and denote as the minority user-group $u^*$. Assuming uniqueness in the minority group definition i.e we assume that there exists an unique user-group which is at a maximum distance to the group $j$, we have that $TV(p_{\phi_u^*}, p_{\phi_j^*}) > TV(p_{\phi^*_{k_{\max}(j)}}, p_{\phi_j^*})$ for $j = j^*$. Using the equality on the right hand side of \eqref{delta_net2}, we obtain
\begin{align}\label{delta_net2_two}
    \|\phi^* - \phi_u^*\|\geq \frac{\epsilon (1 -\eta(u))}{4D} ,
\end{align}
where $\epsilon:=TV(p_{\phi_u^*}, p_{\phi_j^*}) - \max_{k\neq u} TV(p_{\phi^*_{k}}, p_{\phi_j^*}) >0$ (for $j = j^*$). 
%
 Hence proved.

\end{proof}

\section{Proof of Theorem \ref{theorem_2_impossibility}}\label{proof_theoem_1}

\begin{proof}
We can define the alignment gap of RLHF model $  \pi^*_{\text{RLHF}}$ to a specific  user group $u$  as
\begin{align}\label{align_gap_two}
    \text{Align-Gap}(\pi_{\text{RLHF}}): =F_{r_{\phi_u^*}}(\pi_u^*)-F_{r_{\phi_u^*}}(\pi_{\text{RLHF}}).
\end{align}
%
We note that in this specific RLHF setting under the KL-based regularization, the objective $-F_{r_{\phi}}(\pi)$ satisfies strong convexity w.r.t $\pi$ with strong convexity parameter $\mu =1$, hence it holds that
\begin{align}
  \text{Align-Gap}(\pi_{\text{RLHF}}) &\geq \frac{1}{2} \|\pi^* - \pi^*_u\|^2.\end{align}
  Now utilizing that $\log(\pi(\bby/x))$ is Lipschitz continuous with parameter $L_{\pi} = \frac{1}{c}$, under the condition that there exists some $c>0$ such that $\pi(y|x) \geq c$ for all $x,y$, we get
  \begin{align}\label{align_sub1}
      \text{Align-Gap}(\pi_{\text{RLHF}}) 
    & \geq  \frac{1}{2 L_{\pi}} \|\log \pi^* - \log \pi^*_u\|^2. \end{align}
    From the results in \citep{rafailov2023direct}, we can derive an analytical mapping from reward functions to optimal policies for the KL-constrained reward maximization objective as denied in \eqref{eq:RL2} as :
    \begin{align}\label{dpo}
        \pi_r(y\mid x) = \frac{1}{Z(x)}\pi_{\text{ref}}(y\mid x)\exp\left(\frac{1}{\beta}r(y, x)\right)
    \end{align}
    where $\pi_r$ is the optimal policy under the reward $r$ and $Z(x)$ is the partition function given as $Z(x) =\sum_{y}\pi_{\text{ref}}(y\mid x)\exp\left(\frac{1}{\beta}r(y, x)\right)$. Note that such an equivalence is specific to the RLHF problem under the Bradley Terry preference model as shown in \citep{rafailov2023direct}. Next, replacing equation \eqref{dpo} in the equation \eqref{align_sub1}, we get
        \begin{align}\label{rwd_param_norm}
  \text{Align-Gap}(\pi_{\text{RLHF}}) 
        & \geq  \frac{ 1}{2 L_{\pi} \beta^2} \|r_{\phi^*}-r_{\phi_u^*}\|^2 \\ \nonumber
        & =  \frac{1}{2 L_{\pi} \beta^2} \|\langle\Psi, \phi^{*}-\phi^{*}_u \rangle\|^2.
        \end{align}
As stated in \eqref{model_2}, under the linearly parametrized reward function, we have $r_{\phi} (\mathbf{y}, \mathbf{x}) = \langle\phi, \psi({\mathbf{y}}, \mathbf{x})\rangle$, where the parameter $\phi \in \mathbb{R}^d$ and similarly $\psi({\mathbf{y}}, \mathbf{x}) \in \mathbb{R}^d$.  Let $(x,y) \in  \mathbb{R}^n$ and we denote the feature matrix $\Psi \in \mathbb{R}^{n \times d}$ as 
$\begin{bmatrix}
\Psi^T = \Psi(y_1,x_1) & \Psi(y_2,x_2) & \cdots & \Psi(y_n,x_n)
\end{bmatrix}$, replacing in \eqref{rwd_param_norm}, we get the final expression. Next, expanding the norm on the right hand side, we obtain
               \begin{align}
  \text{Align-Gap}(\pi_{\text{RLHF}})  \geq  &  \frac{1}{2 L_{\pi} \beta^2} (\phi^{*} -\phi_u^{*})^T  \Psi^T \Psi   (\phi^{*} -\phi_u^{*}).
        \end{align}
Next we lower-bound the matrix norm of $\Psi^T \Psi \in \mathbb{R}^{d \times d}$ with the minimum eigen value $\lambda_{\psi}$ as
        \begin{align}\label{here_4}
       \text{Align-Gap}(\pi_{\text{RLHF}})  & \geq  \frac{\lambda_{\psi}}{4 L_{\pi} \beta^2} \|\phi^*-\phi_u^*\|^2,
       \end{align}
       where we obtain the lower bound in terms of the reward suboptimality. From the statement of Lemma \ref{theorem_1_impossibility0}, we can lower bound the right hand side in \eqref{here_4} as follows
       \begin{align}
   \text{Align-Gap}(\pi_{\text{RLHF}})  \geq & \frac{ \lambda_{\psi} }{4 L_{\pi} \beta^2}  \frac{\epsilon (1 -\eta(u))}{16D^2},
\end{align}
where $\epsilon:=TV(p_{\phi_u^*}, p_{\phi_j^*}) - \max_{k\neq u} TV(p_{\phi^*_{k}}, p_{\phi_j^*}) >0$.
Hence proved. 

\end{proof}

\section{Additional Details of the Experiments}\label{appendix_experiments}
In this section, we provide additional details of the experiments in Section \ref{experiments}.

\begin{table*}[h]
\caption{Dataset Summary}
\label{dataset_summary}
\begin{center}
\begin{small}
\begin{tabular}{c | c }
\toprule
User Group & Preference Prompt \\
\midrule
P1A & Generate/Choose a response that can be easily understood by an elementary school student.\\
P1B & Generate/Choose a response that only a PhD Student in that specific field could understand. \\
P2A & Generate/Choose a response that is concise and to the point, without being verbose. \\
P2B & Generate/Choose a response that is very informative, without missing any background information.\\
P3A & Generate/Choose a response that is friendly, witty, funny, and humorous, like a close friend.\\
P3B & Generate/Choose a response (that answers) in an unfriendly manner.\\

\end{tabular}
\end{small}
\end{center}
\end{table*}
\begin{figure}[h]
    \centering 
    \subfloat[Testing Distribution on Dataset P1A/P1B]{
        \includegraphics[width=0.5\columnwidth]{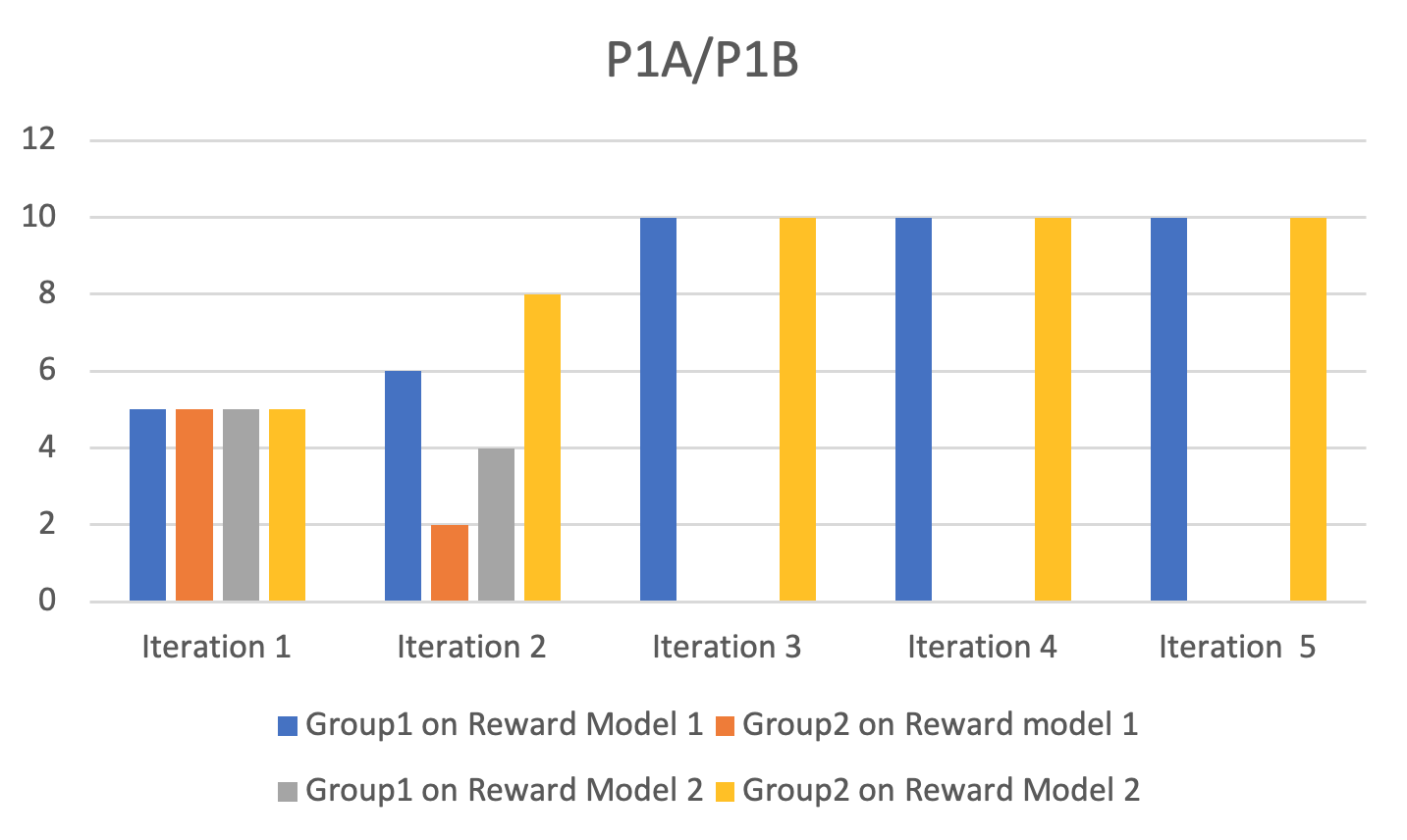}
        \label{fig:P1TestingDistribution}
    }
    \subfloat[Accuracy on Dataset P1A/P1B]{ 
        \includegraphics[width=0.5\columnwidth]{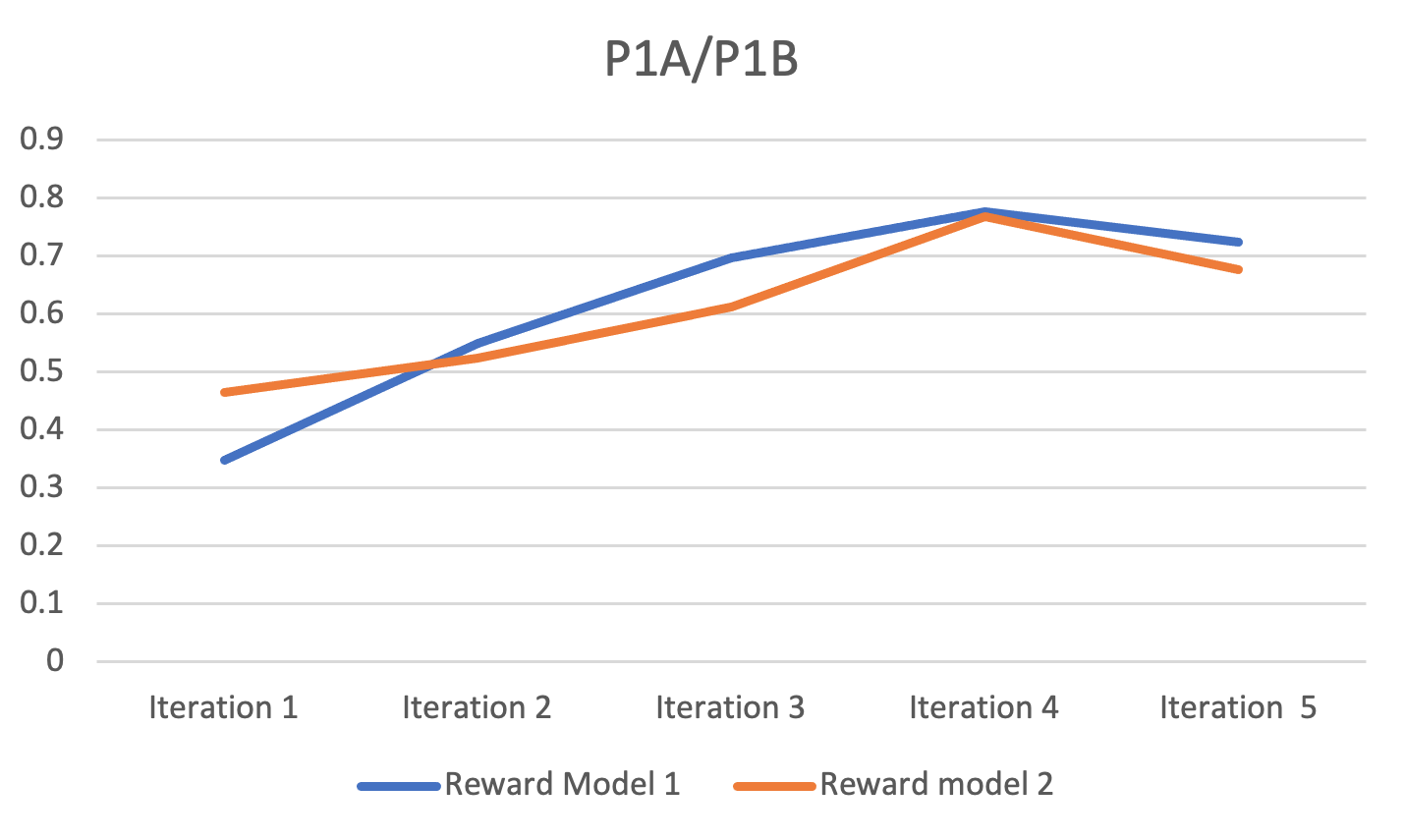}
        \label{fig:P1Accuracy}
    }
    \caption{Results on Dataset P1A/P1B.}
    \label{fig:result1}
\end{figure}

\section{Additional Experiments in Robotics Navigation Tasks}
In this section, we show that the proposed ideas are well extendable to reinforcement learning in general. We show the performance of MaxMin alignment with single reward RLHF on simple gridworld navigation in Figure \ref{fig:enter-label_appendix}.

\begin{figure*}[h]
    \hspace{8mm}\includegraphics[width=1.2\textwidth]{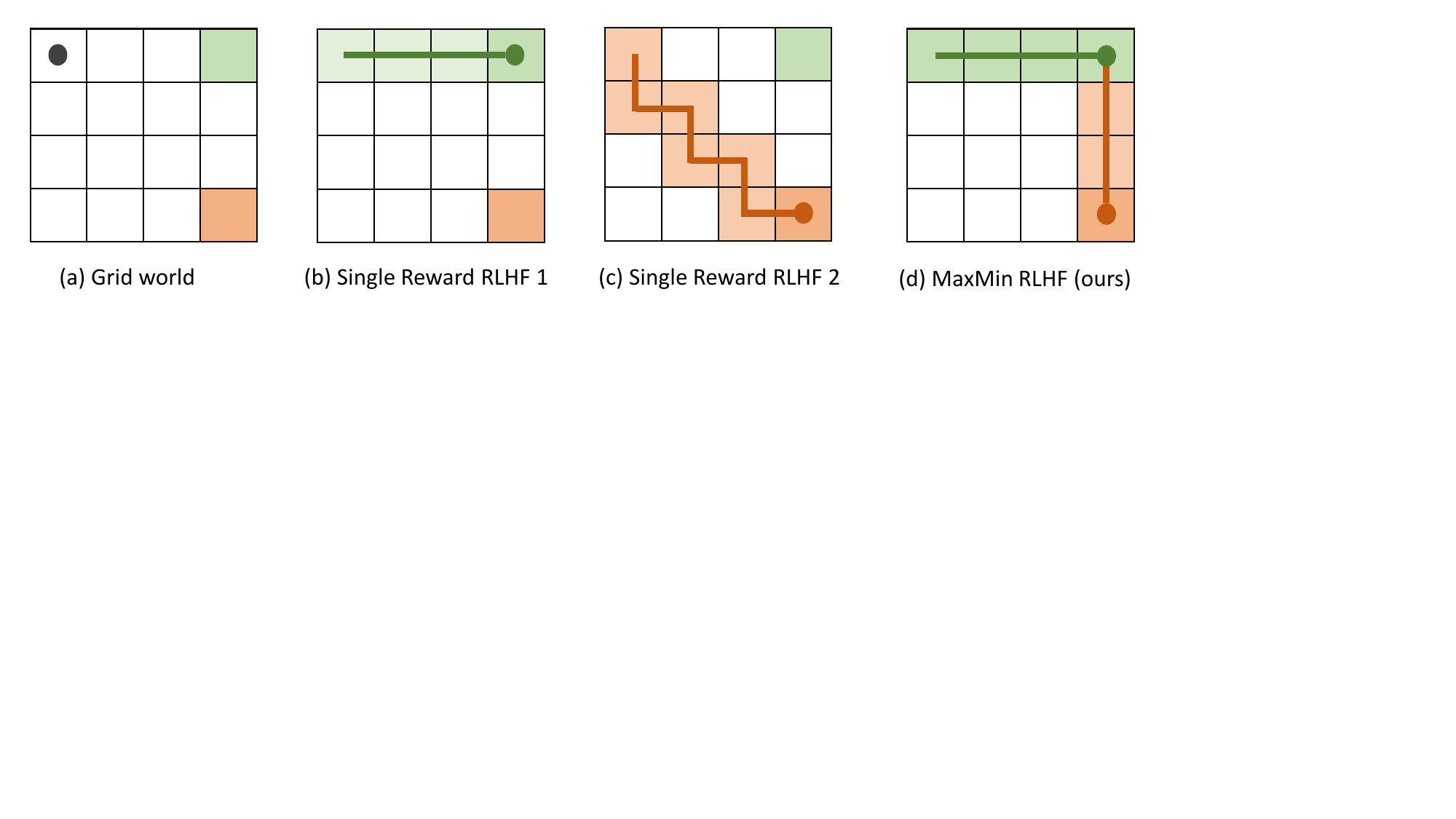}\vspace{-7cm}
     \caption{(a) This figure shows a GridWorld navigation scenario where say a government supported vehicle  which needs to distribute goods among the two groups denoted by \textcolor{green}{green} and \textcolor{orange}{orange} boxes. In the above figure, (b) shows the trajectory when only \textcolor{green}{green} user preferences are considered to decide the vehicle path. (c) shows the trajectory when only \textcolor{orange}{green} user preferences are considered to decide the vehicle path. (d) shows the result for our proposed formulation, where our goal is to maximize the social utility, which makes sure to develop a robust solution to satisfy all user preferences.}
    \label{fig:enter-label_appendix}
\end{figure*}

\end{document}